\documentclass[11pt]{article}
\usepackage[a4paper,total={6.5in,9.5in}]{geometry}
\usepackage[english]{babel}

\usepackage[utf8]{inputenc} 
\usepackage[T1]{fontenc}    
\usepackage[colorlinks=true,linkcolor=blue]{hyperref}      
\usepackage{url}           
\usepackage{booktabs}       
\usepackage{amsfonts}      
\usepackage{nicefrac}      
\usepackage{microtype}
\usepackage{amsthm}
\usepackage{natbib}
\usepackage{centernot}
\usepackage[dvipsnames]{xcolor}

\usepackage{float}
\usepackage{graphics}
\usepackage{subcaption}
\graphicspath{{./figures/}}
\DeclareGraphicsExtensions{.pdf}

\usepackage{algorithm}
\usepackage{algpseudocode}
\usepackage{indentfirst}
\usepackage{mathtools}
\usepackage{amsmath}
\usepackage{tabularx} 
\usepackage{multirow}        
\usepackage{caption}            
\usepackage{bm}
\usepackage{algorithm}
\usepackage{algpseudocode}
\usepackage{amsthm}
\newtheorem{lemma}{Lemma}  
\usepackage{amssymb}
\usepackage[capitalize]{cleveref}

\usepackage{nicefrac}

\usepackage{comment}

\usepackage{amsthm}
\usepackage{graphicx}

\usepackage{float} 
\usepackage{appendix}
\usepackage{float}  

\newcommand{\independent}{\perp\mkern-9.5mu\perp}
\newcommand{\notindependent}{\centernot{\independent}}

\newtheorem{assumption}{Assumption}

\newtheorem{theorem}{Theorem}
\newtheorem{corollary}[theorem]{Corollary}
\newtheorem{proposition}[theorem]{Proposition}

\newtheorem{example}{Example}[section]
\theoremstyle{remark}
\newtheorem{remark}{Remark}
\newcommand{\ci}[3]{#1 \perp\kern-5pt \perp #2 \mid #3}
\theoremstyle{definition}

\title{MinShap: A Shapley-Based Framework for Feature Redundancy}

\author{Chenghui Zheng
\thanks{Department of Statistics, University of Wisconsin - Madison, Madison, WI 53706}
\and 
Garvesh Raskutti
\footnotemark[1]
}

\date{\vspace{-5ex}}

\begin{document}
\maketitle
\begin{abstract}
Shapley values provide a flexible framework for attributing feature contributions to model predictions, but they are not naturally suited for feature selection: a feature may receive a positive attribution even when it is redundant given the remaining variables. In this paper, we introduce \textbf{MinShap}, a general framework for identifying \emph{important} or \emph{non-redundant} features through conditional importance functionals $VI_j^S$. Rather than averaging feature contributions across conditioning sets, MinShap aggregates them using the \emph{minimum}, thereby testing whether a feature remains relevant under every conditioning context. We show that, under a simple \emph{null monotonicity} condition, the minimum aggregation exactly characterizes feature redundancy and yields a principled feature selection criterion. This perspective provides a unified framework for statistical feature selection and representation-based interpretability while retaining the stability advantages of Shapley-style aggregation. We develop scalable algorithms with statistical guarantees, establish connections to multiple-testing procedures, and demonstrate through theory and experiments that MinShap produces more accurate and stable feature selection than existing model-agnostic approaches.
\end{abstract}

\section{Introduction}

Feature attribution and feature selection are often treated as closely related problems, yet they address fundamentally different scientific questions. Feature attribution seeks to quantify \emph{how much} a feature contributes to a model prediction, whereas feature selection seeks to determine whether a feature is \emph{redundant} once all remaining information is available. A feature may contribute substantially to prediction while nevertheless be redundant because its information is already contained in other variables. Distinguishing between contribution and redundancy is therefore essential for constructing reliable feature selection procedures.

Shapley values~\cite{shapley_value_1953} provide one of the most widely used frameworks for feature attribution in interpretable machine learning~\cite{sundararajan_many_2020,janzing_feature_2020,chen_algorithms_2023,nazir_survey_2025}. Given a prediction function
\[
f:\mathcal{X}^p\rightarrow\mathcal{Y},
\]
the Shapley value of feature $X_j$ is
\begin{equation}
\label{def1}
\phi_j
=
\sum_{S\subseteq[p]\setminus\{j\}}
\frac{1}{p}
\binom{p-1}{|S|}^{-1}
\left\{
V(f(X_{S\cup\{j\}}))
-
V(f(X_S))
\right\},
\end{equation}
where $V(\cdot)$ denotes a value functional over feature subsets. One of the strengths of the Shapley framework is that different choices of the value function recover statistical, causal, functional, and other notions of feature importance~\cite{fryer2021model,plischke2021computing}.

One of the key ingredient of the Shapley framework is its aggregation operator. Shapley values average the marginal contribution of a feature across all subsets (or equivalently all feature orderings~\cite{castro_polynomial_2009}), producing attribution scores that are stable and robust, particularly when predictors are correlated. However, averaging is naturally aligned with measuring \emph{average contribution}, not \emph{selection} or \emph{redundancy}. Consequently, a feature may receive a positive Shapley value even when it is conditionally redundant given the remaining variables~\cite{kumar_problems_2020,huber_conditional_2024}; that is,
\[
\phi_j>0,
\]
even though the feature provides no additional information once all remaining variables are observed.

Existing model-agnostic statistical feature selection methods instead target conditional importance directly. Popular examples include Leave-One-Covariate-Out (LOCO)~\cite{lei_distribution-free_2018} and the Generalized Covariance Measure (GCM)~\cite{shah_hardness_2020}, both of which are designed to identify features that remain informative after conditioning on the remaining covariates. While these approaches naturally target redundancy, their performance can deteriorate in the presence of strong feature dependence, higher-order interactions, or inaccurate nuisance estimation~\cite{wang_precision_2019,kaneko_crossvalidated_2022,maia_polo_conditional_2023,verdinelli_decorrelated_2024}. Thus, existing methods exhibit a fundamental trade-off: Shapley values provide stable attribution but do not directly identify significant and redundant features, whereas conditional importance methods target the correct notion of relevance but can be unstable.

In this paper we revisit the Shapley framework from a different perspective. Rather than asking whether the value function should be changed, we ask whether the \emph{aggregation operator} should depend on the scientific question being addressed. While averaging naturally quantifies average contribution, feature selection requires determining whether a feature remains relevant under \emph{every} conditioning context.

Our first major contribution is a general framework, called \emph{MinShap}, obtained by replacing the averaging operator in the Shapley value with the minimum. We introduce a simple \emph{null monotonicity} condition under which this minimum aggregation exactly characterizes feature redundancy. The resulting framework applies to arbitrary value functions and therefore unifies statistical feature selection and representation-based interpretability within a common aggregation principle. Under null monotonicity, MinShap transforms the Shapley framework from an attribution method into a principled feature selection framework while retaining the stability benefits that arise from aggregating information across multiple conditioning sets.

Our second major contribution specializes this general framework to model-agnostic statistical feature selection. Using predictive risk as the value function, we show that null monotonicity follows under standard assumptions for graphical models, develop scalable estimation and inference procedures with Type~I error guarantees, and establish connections between MinShap, Max-$p$ aggregation, and partial conjunction hypothesis testing~\cite{benjamini_screening_2008,williamson_general_2023}. Through simulation studies and real data analyses, we demonstrate that these procedures provide accurate and stable feature selection and compare favorably with existing model-agnostic approaches.

\subsection{Our contributions}

The contributions makes two complementary and related contributions.

\paragraph{Adapting the Shapley value framework for feature selection.}

\begin{itemize}
\item We introduce MinShap, a general feature selection framework obtained by replacing the averaging operator in the Shapley value with the minimum.

\item We identify the null monotonicity condition under which minimum aggregation exactly characterizes feature redundancy.

\item We show that the framework applies to both statistical feature selection and representation-based interpretability through appropriate choices of the value function.
\end{itemize}

\paragraph{Applying the general framework to model-agnostic statistical feature selection.}

\begin{itemize}
\item We specialize MinShap to statistical feature selection using predictive risk as the value function and show that null monotonicity follows under standard graphical assumptions.

\item We develop scalable MinShap algorithms with Type~I error guarantees and establish connections to Max-$p$ aggregation and partial conjunction hypothesis testing.

\item Through simulations and real data analyses, we demonstrate that MinShap and the associated multiple-testing procedures provide accurate and stable feature selection compared with existing model-agnostic approaches including LOCO, GCM, and the Lasso.
\end{itemize}
\subsection{Related work}

\paragraph{Shapley values and feature selection.}
The Shapley value, originally introduced in cooperative game theory~\cite{shapley_value_1953}, has become one of the most widely used frameworks for feature attribution in machine learning~\cite{williamson_efficient_2020,tripathi_feature_2020,covert_understanding_2020,patel_game-theoretic_2021}. Existing SHAP-based feature selection methods typically rank features using global summaries such as mean absolute SHAP values~\cite{lundberg_unified_2017}, followed by a user-specified threshold~\cite{tripathi_interpretable_2020,marcilio_explanations_2020,jothi_predicting_2021}. While simple and model-agnostic, these approaches generally lack formal statistical guarantees and inherit the semantics of Shapley values, namely measuring average feature contribution rather than conditional dependence. More recent methods such as BorutaSHAP~\cite{keany_borutashap_2020} and SHAP-XRT~\cite{teneggi_shap-xrt_2023} improve thresholding and statistical inference, but require either repeated retraining or conditional randomization. In contrast, MinShap modifies the aggregation operator itself by replacing the average with the minimum. Under a null monotonicity condition, this directly characterizes feature redundancy while retaining the aggregation of Shapley values.

\paragraph{Model-agnostic feature selection.}
Model-agnostic feature selection methods seek to identify features that remain informative independently of the underlying predictive model. Popular examples include Leave-One-Covariate-Out (LOCO)~\cite{lei_distribution-free_2018}, which compares prediction error before and after removing a feature, and the Generalized Covariance Measure (GCM)~\cite{shah_hardness_2020}, which formulates feature selection as a conditional independence testing problem. Although these methods directly target conditional relevance, their performance can deteriorate under strong feature dependence, higher-order interactions, or inaccurate nuisance estimation~\cite{wang_precision_2019,kaneko_crossvalidated_2022,maia_polo_conditional_2023,verdinelli_decorrelated_2024}. Embedded approaches such as the Lasso~\cite{tibshirani_regression_1996} are computationally efficient but are inherently model-specific rather than model-agnostic.

\paragraph{Multiple testing and stability selection.}
The statistical MinShap framework admits a natural multiple-testing interpretation through the aggregation of conditional importance scores across conditioning sets. This connects our methodology to stability selection~\cite{meinshausen_stability_2010,shah_variable_2013}, which aggregates feature rankings across resampled datasets. In contrast, MinShap aggregates across conditioning sets or feature orderings induced by the Shapley framework, leading to a single theoretically justified test statistic and threshold under the null monotonicity assumption.

\section{Background}

\subsection{General setup}

Consider the supervised learning setting with observations
\[
\{(X^{(i)},Y^{(i)})\}_{i=1}^n,
\qquad
X^{(i)}=(X_1^{(i)},\ldots,X_p^{(i)})\in\mathcal X^p,
\]
and prediction function
\[
f:\mathcal X^p\rightarrow\mathcal Y.
\]

Throughout the paper we make no assumptions on either the data-generating distribution or the prediction function, allowing the framework to apply to statistical, causal, and representation-based notions of feature importance.

For a subset $S\subseteq[p]=\{1,\ldots,p\}$, let $X_S=\{X_j:j\in S\}$ denote the corresponding feature subset, and let
\[
X_{-S}=X_{S^c},
\qquad
X_{-j}=X_{[p]\setminus\{j\}}.
\]

\subsection{Shapley values and feature importance}

Shapley values~\cite{shapley_value_1953} provide an axiomatic framework for allocating the value of a coalition among its members and have become one of the most widely used approaches for feature attribution in interpretable machine learning~\cite{sundararajan_many_2020,janzing_feature_2020,chen_algorithms_2023,nazir_survey_2025}. Let
\[
V:2^{[p]}\rightarrow\mathbb R
\]
denote a value function, typically of the form
\[
V(S)=V(f(X_S)).
\]
The Shapley value for feature $X_j$ is
\begin{equation}
\label{def1}
\phi_j
=
\sum_{S\subseteq[p]\setminus\{j\}}
\frac1p
\binom{p-1}{|S|}^{-1}
\Bigl(
V(S\cup\{j\})-V(S)
\Bigr),
\end{equation}
which is the unique attribution satisfying the standard Shapley axioms~\cite{shapley_value_1953}.

It is convenient to define the conditional feature importance
\[
VI_j^S
:=
V(S\cup\{j\})-V(S),
\]
so that
\[
\phi_j
=
\sum_{S\subseteq[p]\setminus\{j\}}
\frac1p
\binom{p-1}{|S|}^{-1}
VI_j^S.
\]

Equivalently, Shapley values admit the permutation representation~\cite{castro_polynomial_2009}
\begin{equation}
\label{def2}
\phi_j
=
\frac1{p!}
\sum_{\pi\in\Pi(p)}
VI_j^\pi,
\end{equation}
where $VI_j^\pi$ denotes the marginal contribution of feature $j$ under permutation $\pi$. This interpretation views Shapley values as averaging conditional feature importance over all feature orderings and motivates the MinShap framework developed in the next section.

\subsection{Feature selection}

Our goal is to identify the subset of important or non-redundant features. Given a value function $V$, define the optimal feature subset by
\[
\widetilde S
\in
\arg\min_{|S|}
\{V(S)=V([p])\}.
\]

A particularly important instance is statistical feature selection, where
\[
V(S)
=
\sup_{f_S}
-
\mathbb E
\bigl[
\ell(Y,f_S(X_S))
\bigr].
\]

Under standard assumptions on the loss and hypothesis class, predictive sufficiency is equivalent to conditional independence.

\begin{lemma}[Predictive sufficiency]
\label{lem:predictive-sufficiency}
Suppose $\ell$ is a strictly proper loss and the Bayes predictor belongs to the hypothesis class. Then
\[
V(S)=V([p])
\iff
Y\independent X_{-S}\mid X_S.
\]
Consequently,
\[
\arg\min_{|S|}
\{V(S)=V([p])\}
=
\arg\min_{|S|}
\{Y\independent X_{-S}\mid X_S\}.
\]
\end{lemma}
We defer the proof to the Appendix \ref{pf_lemma_Predictive_sufficiency}. Searching over all $2^p$ subsets is computationally infeasible, motivating the feature-wise definition
\[
S^*
=
\left\{
j:
VI_j^{-j}>0
\right\},
\]
which, under standard uniqueness assumptions~\cite{statnikov_algorithms_2013}, coincides with the minimal predictive subset. Throughout the remainder of the paper we adopt this feature-wise formulation.

\subsection{Why Shapley values are insufficient for feature selection}

Equation~\eqref{def2} shows that Shapley values average feature importance over all conditioning sets. While this averaging improves stability, it does not characterize conditional redundancy.

In particular, a feature may satisfy
\[
VI_j^{-j}=0,
\]
where $VI_j^{-j}:=VI_j^{\{1,2,...,p\}/\{j\}}$, which means that it provides no additional information once all remaining features are observed, yet still receive a positive Shapley value because it contributes under other conditioning sets. Consequently,
\[
\phi_j=0
\Longrightarrow
VI_j^{-j}=0,
\qquad
\text{but}
\qquad
VI_j^{-j}=0
\not\Longrightarrow
\phi_j=0.
\]

The following example illustrates this phenomenon.

\begin{example}\label{ex:DAG}
    DAG: $X_1 \rightarrow X_2 \rightarrow X_3 \rightarrow  Y$, $X_2 = X_1 + \gamma; X_3 = X_2+\delta; Y=X_3+\epsilon$, where $X_j\sim N(0,1), j \in \{1,2,3\}$ and $\epsilon, \delta, \gamma \sim N(0,1)$ are independent of $X$.

\begin{table}[H]
  \centering
  \caption{Marginal contributions for features in chain DAG example.}
  \label{tab:example}
  \begin{small}
  \begin{tabular}{lccc}
    \toprule
    Permutation & $VI_{1}^{\pi}$ & $VI_{2}^{\pi}$ & $VI_{3}^{\pi}$ \\
    \midrule
    $[X_1, X_2, X_3]$ & 1 & 1 & 1 \\
    $[X_1, X_3, X_2]$ & 1 & 0 & 2 \\
    $[X_2, X_1, X_3]$ & 0 & 2 & 1 \\
    $[X_2, X_3, X_1]$ & 0 & 2 & 1 \\
    $[X_3, X_1, X_2]$ & 0 & 0 & 3 \\
    $[X_3, X_2, X_1]$ & 0 & 0 & 3 \\
    \bottomrule
  \end{tabular}
  \end{small}
\end{table}

The Shapley values for $(X_1, X_2, X_3)$ are:
\[
\phi_{1} = \frac{2}{6}, \quad 
\phi_{2} = \frac{5}{6}, \quad 
\phi_{3} = \frac{11}{6}.
\]
\end{example}

Even ranking features according to their Shapley values cannot generally recover the optimal feature subset (see, for example, the Markov boundary examples in~\cite{ma_predictive_2020,fryer_shapley_2021}). This limitation motivates replacing the averaging operator in~\eqref{def2} by an aggregation that directly targets conditional redundancy.

\section{Null Monotonicity and MinShap Algorithm in the Oracle Setting}\label{sec:minshap}
In this section we introduce our MinShap framework in the oracle population setting. First we motivate the algorithm by discussing and imposing the null monotonicity assumption.

\subsection{Motivating MinShap via the null monotonicity condition}

\begin{assumption}[Null monotonicity condition] \label{ass_monotonicity}
For every triple $(j, S, T)$ where $j \not \in T$ and $S \subseteq T$,
$$
VI_j^S = 0 \Rightarrow VI_j^T = 0
$$
\end{assumption}
In words, the null monotonicity condition ensures that if $X_j$ is not important in the presence of features $X_S$, it will remain unimportant if we add more features. This is a significantly weaker condition than requiring full monotonicity,
\[
VI_j^T \leq VI_j^S, \qquad S\subseteq T,
\]
which asserts that adding additional features can only decrease the importance
of feature $X_j$. Null monotonicity imposes only the much weaker requirement that
once a feature has zero conditional importance, it remains unimportant after
conditioning on additional features.


\begin{theorem}\label{thm:main}
 Under Assumption~\ref{ass_monotonicity},
$$
\phi_{j}^{(min)}:=\min_{S} VI_{j}^S = 0 \Leftrightarrow VI_{j}^{-j} = 0 
$$
or equivalently
$$
 VI_{j}^{-j} > 0\;\; \mbox{if and only if}\;\; VI_{j}^{S} > 0 \mbox{, for all}\;\;S.
$$    
\end{theorem}
\begin{proof}
By the null monotonicity assumption, if there exists an $S \subseteq \{X_1, X_2,...,X_p\}/\{X_j\}$ such that $VI_{j}^S = 0$ then $VI_{j}^{-j} = 0$ which proves the forward/right direction. For the reverse/left direction, simply choosing $S = \{X_1, X_2,...,X_p\}/\{X_j\}$ leads to the result.
\end{proof}

\subsection{Statistical risk, feature selection and null monotonicty}

For the case of statistical risk, we prove null monotonicity using the DAG modeling language~\cite{pearl_causal_2009} for multi-variate distributions. In particular value function $V(\cdot)$ takes the form:
$$
V(S) := V((f_S(X_S))) = \mathbb{E}[-\ell(Y, f_S(X_S))]
$$
for a loss function $\ell(\cdot,\cdot)$. A DAG represents directed relationships among variables, where nodes/vertices correspond to variables (e.g. features or outcome). DAGs define d-separation rules (see e.g.~\cite{spirtes_causation_2001,pearl_causal_2009}) and link them to conditional independence statements in a probability distribution via the faithfulness assumption listed below. 
\begin{assumption}\label{ass:faithful}
     (Faithfulness). The probability distribution $P$ is faithful to DAGs, such that the conditional independence relationships that hold in $P$ are exactly those implied by the DAGs via d-separation.   
\end{assumption}
As a consequence of the faithfulness assumption, we are able to claim that any independence relations in the data are caused by the underlying structure of the graph that generates it, so there is no accidental conditional independence arising from random coincidence or special cancellations. Without the faithfulness assumption, a Shapley summand could be zero because of accidental cancellation, even though there is still an active path in the graph. We also need an additional reverse causation assumption:

\begin{assumption}\label{ass:revser_cau}
      (No reverse causation). The outcome
$Y$ in DAGs does not causally influence any covariate, such that $Y$ is not ancestor of any vertex $X_j$ in DAGs.
\end{assumption}
No reverse causation is standard in machine learning because we assume that the response $Y$ is a function of the features $(X_1, X_2,...,X_p)$ and not the other way around. This assumption rules out $Y$ being a collider.

\begin{lemma}\label{lemma}
Under Assumptions ~\ref{ass:faithful} and ~\ref{ass:revser_cau}, the null monotonicity condition is satisfied. More precisely,
$$
Y \independent  X_j |X_S \Rightarrow Y\independent X_j |X_T
$$
for any $S \subseteq T \subseteq \{1,2,...,p\}/\{j\}$.
\end{lemma}
\begin{proof}
The proof follows from the fact that since $S \subseteq \{1,...,p\}\backslash\{j\}$, using the language of DAG models~\cite{pearl_causal_2009, peters_causal_2014}, conditioning on an additional node $X_k$ with $k \in T$ can only create a path to $Y$ if it is a collider $X_j \rightarrow X_k$ and $Y \rightarrow X_k$. But under the no data leakage assumption, there can be no edge from $Y$ to the covariates. Therefore conditioning on an additional node cannot lead to conditional dependence. 
\end{proof}

\subsection{Permutation/ordering perspective of MinShap}

While the original definition is from the weighted aggregation over subsets~\cite{shapley_value_1953}, there is an equivalent definition from a permutation/ordering perspective following on from the permutation/ordering perspective of Shapley values~\cite{castro_polynomial_2009} described earlier. 

For an ordering $\pi \in \Pi(p)$, $VI_{j}^\pi$ is the Shapley value summand for feature $j$ captures the additional contribution of feature $j$ to the target given ordering set $[\pi]_{j-1}$. That is to say, under a true ordering $\pi^*$, 
$$
VI_{j}^{\pi^*} = 0 \Leftrightarrow Y \independent X_{j}|X_{{[\pi^*]_{j-1}}}.
$$
Since $\pi^*$ is typically unknown, if we assume $(X_1,...,X_p, Y)$ follow a recursive structural DAG model where $Y$ is a terminal node in the DAG~\cite{pearl_causal_2009, peters_causal_2014}, then we have the following result:
\begin{theorem}
\label{thm:ordering}
Suppose the conditional importance functional satisfies the null monotonicity
condition. Then
\[
\min_{\pi\in\Pi(p)} VI_j^\pi =0
\]
if and only if
\[
VI_j^{\{1,\ldots,p\}\setminus\{j\}}=0.
\]
Equivalently,
\[
VI_j^{\{1,\ldots,p\}\setminus\{j\}}>0
\quad\Longleftrightarrow\quad
VI_j^\pi>0
\quad\text{for every }\pi\in\Pi(p).
\]
\end{theorem}
\begin{proof}
Since each Shapley summand $VI_{j}^\pi$ is non-negative and each permutation $\pi$ can also be seen as a subset $S$, $\min_{\pi \in \Pi(p)} VI_{j}^{\pi} = 0$ is equivalent to the statement that there exists a valid subset $S$ such that $VI_{j}^{S} = 0$. Thus, under the monotonic null both assumptions hold.
\end{proof}
This permutation/ordering perspective leads to naturally faster algorithms in determining whether $\min_{\pi\in\Pi(p)} VI_j^\pi =0$ because if there exists any ordering $\tilde{\pi}$ such that $VI_j^{\tilde{\pi}} =0$, then every ordering that places $j$ later will also satisfy $0$ importance and we can discard many orderings.

\subsection{Monotonicity in mechanistic interpretability settings}
The focus of this paper is classical feature selection in statistics. However, the MinShap algorithm and null monotonicity assumption applies much more broadly. Here we consider the case where
$$
V(S) := V(f_S(X)) = \mathbb{E}[f_S(X)].
$$

In particular, modern mechanistic interpretability seeks to explain neural networks in terms of
intermediate computational units, such as concepts, sparse latent features, attention heads, or circuits, rather than directly in terms of the input
variables. Examples include concept-based methods such as sparse autoencoder (SAE)
representations~\cite{bricken_towards_2023, cunningham_sparse_2023} and recent
circuit-level analyses of transformer models~\cite{elhage_mathematical_2021,
wang_interpretability_2023}. These methods differ in how explanatory units are constructed, but all can be viewed as restricting a learned representation to a
subset of explanatory components.

Let $\mathcal U=\{u_1,\ldots,u_m\}$ denote a collection of explanatory units.
Depending on the algorithm/method, $u_j$ may represent an input concept, an SAE
feature, an attention head, or a computational circuit. For every subset
$S\subseteq\mathcal U$, let
\[
R_S:\mathcal H\rightarrow\mathcal H
\]
denote a \emph{representation restriction operator}, where $\mathcal H$
denotes the representation space. The restricted predictor is
\[
f_S(x)=g(R_S(h(x))),
\]
where $h(x)$ denotes an internal representation and $g$ maps the restricted
representation to the model output.

We define the conditional importance of explanatory unit $u_j$ by
\[
VI_j^S
=
\mathbb E\!\left[
f_{S\cup\{j\}}(X)-f_S(X)
\right] = \mathbb E\!\left[
g(R_{S\cup j}(h(X))) -g(R_S(h(X)))
\right]
\]  
This measures the additional contribution of explanatory unit $u_j$ once the
units in $S$ are already available.

Unlike the statistical setting, null monotonicity is not derived from conditional independence. Instead, it follows from assumptions on the structure of the learned representation. Null monotonicity in this setting arises from structural properties of the restriction operators rather than probabilistic conditional independence.

\begin{proposition}
If the \emph{representation restriction operator} $R_S$ satisfies 
$$
VI_j^S = 0 \Leftrightarrow R_{S \cup \{j\}} = R_S
$$
for all $S$ and
$$
R_{S \cup \{j\}} = R_S \Rightarrow R_{T \cup \{j\}} = R_T
$$
for all $S \subseteq T$, then
\[
VI_j^S = 0 \;\Rightarrow\; VI_j^T = 0
\quad \text{for } S \subseteq T,
\]
and hence null monotonicity holds.
\end{proposition}
Several existing interpretability methods can be viewed as particular choices of
the restriction operator $R_S$.

\paragraph{Examples of representation restriction operators.}
The general framework encompasses several existing paradigms in mechanistic interpretability through different choices of the restriction operator $R_S$.

\begin{itemize}
\item \textbf{Sparse autoencoder (SAE) representations.}
Let
\[
z(x)=E(h(x))
\]
denote the sparse latent representation produced by an encoder $E$, and let
$D$ denote the corresponding decoder. For
$S\subseteq\{1,\ldots,m\}$, define the diagonal masking matrix
\[
M_S
=
\operatorname{diag}(\mathbf 1_{\{j\in S\}}),
\]
which sets latent coordinates outside $S$ equal to zero. The restriction
operator is
\[
R_S(h)
=
D(M_SE(h)),
\]
and the restricted predictor becomes
\[
f_S(x)
=
g(D(M_SE(h(x)))).
\]
Consequently,
\[
VI_j^S
=
\mathbb E\!\left[
f_{S\cup\{j\}}(X)-f_S(X)
\right]
\]
quantifies the additional contribution of sparse latent feature $j$ given the
latent features in $S$. Since sparse autoencoders are trained to induce approximately sparse linear decompositions of the representation space with features behaving approximately additively in the latent basis\cite{bricken_towards_2023,cunningham_sparse_2023}, approximate representation faithfulness assumption is particularly natural in this setting. Consequently,
for $S \subseteq T$
\[
VI_j^S \approx 0
\quad\Longrightarrow\quad
VI_j^T \approx 0,
\]
which is precisely the null monotonicity condition underlying the MinShap
framework. Deviations from null monotonicity may arise due to residual
feature interactions or imperfect disentanglement of the learned latent
representation.
\item \textbf{Circuit-level interpretability.}
Suppose the network consists of computational modules
\[
\mathcal U=\{u_1,\ldots,u_m\},
\]
where each $u_j$ may represent an attention head, multi-layer perceptron block, or identified
computational circuit
\cite{elhage_mathematical_2021,wang_interpretability_2023}. For
$S\subseteq\mathcal U$, define binary gating variables
\[
\gamma_j
=
\begin{cases}
1,&u_j\in S,\\
0,&u_j\notin S,
\end{cases}
\]
and let $R_S$ denote the operation that replaces every computational unit
outside $S$ by the zero map,
\[
u_j(h)
\longmapsto
\gamma_j\,u_j(h).
\]
Circuits aims to identify computational modules that perform relatively localized and reusable functions
\cite{elhage_mathematical_2021,wang_interpretability_2023}. Under the assumption that these modules behave approximately independently under ablation, if restoring a circuit has no effect given an active set $S$, then enlarging the active subnetwork is not expected to create a new contribution for that circuit. The null monotonicity assumption is most plausible when the extracted circuits are intended to represent relatively modular computational units. Violations of null monotonicity may arise when circuits exhibit strong redundancy, nonlinear interactions, or dynamic reconfiguration under ablation.

\end{itemize}

\section{MinShap and Related Algorithms for Statistical Feature Selection}

In this section we focus specifically on the statistical setting where
$$
V(S):=V(f_S(X_S)) = \mathbb{E}[-\ell(Y, f_S(X_S))].
$$

Now, let's present the main result about how to select the significant features by MinShap algorithm. Denote $V(f_n(X))$ as the estimator of $V(f(X))$, and ${\bf VI_{n,j}} = (VI_{n,j}^{\pi_1},...,VI_{n,j}^{\pi_K}) \in \mathbb{R}^K$ as the vector of $K$ estimators of feature importance statistics computed from sampled permutations, and the order statistics
$$
VI_{n,j}^{(1)}  \leq ... \leq VI_{n,j}^{(K)}.
$$ We define the MinShap test  statistic as $\phi_{n,j}^{(min)} := VI^{(1)}_{n,j}$. The MinShap algorithm resides in the idea that if $VI_{j}^{\pi} > 0$ for all $\pi$, we reject the null that $X_j$ and $Y$ are conditionally independent. Define $\sigma_{n,j}^{2,\pi}$ as the estimator of $\sigma^{2,\pi}_{j}  :=\mbox{Var}(\epsilon^2_{[\pi]_{j-1} \cup \{j\}} - \epsilon_{[\pi]_{j-1}}^2)$, where $\epsilon_{[\pi]}:=Y-f(X_{[\pi]})$. $\sigma_{n,j}^{2,(k)}$ is the associated order statistics, $k \in [K]$. We now formally state the validity of this test.
\begin{theorem}\label{main_thm}
     If assumptions \ref{ass:faithful} and \ref{ass:revser_cau} hold, then under the null hypothesis $H_0: \min_{\pi\in \Pi(p)} VI_{j}^{\pi} = 0 \Leftrightarrow Y \independent X_j |X_{-j}$, the MinShap $VI_{n,j}^{(1)}$ estimated with sufficiently large $K = O(|S^*|)$ ($|S^*|$ is the size of optimal subset/ the number of significant features) permutations is valid, ie: $\mathbb{P}(VI_{n,j}^{(1)} >t_j|H_0) < \alpha$, for all $\alpha \in [0,1]$ with threshold $t_j>\sqrt{-2\ln(\alpha)\sigma_{n,j}^{2,(1)}}$.
\end{theorem}

\begin{remark}
The MinShap algorithm works well in filtering out the insignificant feature $j$ when $|S^*|$, the number of significant features, is relatively small. Consider the worst case when feature $j$ carries partial information of all significant features, the chances of selecting the correct true permutation order where $j$ is placed after all significant features will be high if $|S^*|$ is small.
\end{remark}
\begin{remark}
The Type II error is driven by the number of permutation $K$ due to the independent multiple hypothesis testing procedure. Thus, fewer permutation $K$ maximizes the power of MinShap. For further details, see Appendix \ref{app:type2}.
\end{remark}

Based on the threshold derived from Theorem~\ref{main_thm}, we state Algorithm \ref{alg:minshap}, which implements MinShap for statistical feature selection.
\begin{algorithm}[h]
\caption{MinShap algorithm for feature selection}
\label{alg:minshap}
\begin{algorithmic}[1]
\Require Dataset $(X,Y)$; predictiveness measure $V$; number of features $p$; number of permutations $K$; significance level $\alpha$

\State Fit null model $f_{n,\emptyset}$; set $V_{\emptyset}\leftarrow V(f_{n,\emptyset})$ and residuals $\mathbf e^2_{\emptyset}$
\State Initialize storage arrays $\{\phi^{\pi_k}_{n,j}\}$ and $\{\sigma^{2,\pi_k}_{n,j}\}$ for all $j\in[p],\,k\in[K]$

\For{$k=1,\ldots,K$}
  \State Sample a random permutation $\pi_k$
  \State Set $V_{\mathrm{cur}} \leftarrow V_{\emptyset}$ and $\mathbf e^2_{\mathrm{cur}} \leftarrow \mathbf e^2_{\emptyset}$
  \For {$j=1,\ldots,p$}
    \State Set predecessor set $\mathcal P^{\pi_k}_{j-1} \leftarrow \{\, i\in[p]: \pi_k(i)<\pi_k(j)\,\}$
    \State Fit model $f_{n,\mathcal P^{\pi_k}_{j-1}\cup\{j\}}$ and compute $V_{\mathrm{new}} \leftarrow V(f_{n,\mathcal P^{\pi_k}_{j-1}\cup\{j\}})$ and residuals $\mathbf e^2_{\mathrm{new}}$
    \State $VI^{\pi_k}_{n,j} \leftarrow V_{\mathrm{new}}-V_{\mathrm{cur}}$
    \State $\phi^{\pi_k}_{n,j}\leftarrow VI^{\pi_k}_{n,j}$\; ; \; $\sigma^{2,\pi_k}_{n,j}\leftarrow Var(\mathbf e^2_{\mathrm{new}}-\mathbf e^2_{\mathrm{cur}})/n$
    \State $V_{\mathrm{cur}}\leftarrow V_{\mathrm{new}}$\; ; \;$\mathbf e^2_{\mathrm{cur}} \leftarrow \mathbf e^2_{\mathrm{new}}$
  \EndFor
\EndFor

\For{$j=1,\ldots,p$}
  \State $\phi^{(\min)}_{n,j} \leftarrow \min_{k\in[K]} \phi^{\pi_k}_{n,j}$
  \State $t_{j} \leftarrow \sqrt{-2\log(\alpha)\cdot \sigma^{2,(1)}_{n,j}}$
  \State Reject $H_{0,j}$ if $\phi^{(\min)}_{n,j}\ge t_j$
\EndFor

\State \textbf{Return} $\phi^{(\min)}_{n,j}$, $t_{j}$, and decisions for all $j\in[p]$
\end{algorithmic}
\end{algorithm}

\subsection{Computational complexity}
The computation complexity of MinShap is the same as permutation-based Shapley value. The permutation approach is preferred over the subset approach as it reuses the fitted model, reducing the computational time to half of the subset approach. There are a number of ways to lower the computation complexity of computing MinShap. First, for different feature $j$ and different permutation $\pi \in \Pi$, the computational jobs can be run in parallel and aggregate results at the end. In this way, the computation complexity can be reduced to $O(p)$ at the price of higher storage space. Second, the number of permutation order set we used to approximate Shapley value can be proportional to the number of significant features $|S^*|$ not the number of features $p$, as we discussed in Theorem \ref{main_thm}. The classic Shapley value approximation via Monte Carlo sampling ~\cite{castro_polynomial_2009, rozemberczki_shapley_2022} requires $O(p^2)$ in time complexity where we have two nested loops: outer over permutations $p$, and 
inner over features in the permutation $p$. Now, the time complexity can be reduced to $O(|S^*|p)$. Third, to save time from retraining models, we can either use \textit{Dropout} method where the $j^{th}$ feature is replaced by its corresponding marginal mean and then reinserted into the pre-trained full model $f_0$, or the Lazy-training~\cite{gao_lazy_2022} by gradient computation in neural network set up to estimate reduced model parameter. 

\subsection{Extension to partial conjunction hypothesis testing}
\label{Sec_pvalue}

MinShap can naturally be connected to the maximum p-value in multiple conditional independence tests. The construction of p-value test for feature importance is derived from the framework of Williamson et al.\cite{williamson_general_2023} that the estimator $V(f_n(X))$ of value function $V(f(X))$ is asymptotically normal and efficient if several standard conditions related to smoothness, identifiability, and finite moments hold. Thus, the result can immediately be applied to variable importance estimator $VI_{n,j}^{\pi} $ by continuous mapping theorem.

 \begin{theorem}[Re-formulated from~\cite{williamson_general_2023}]\label{thm:VI}
  If assumptions in~\cite{williamson_general_2023} hold, then each estimator of Shapley value summand $VI_{n,j}^{\pi} $ for $j \in [p]$ under permutation order $\pi$, has the following result:\[\sqrt{n}\big[VI_{n,j}^{\pi} - VI_{j}^{\pi}\big] \overset{d}{\rightarrow} N(0, \sigma^{2,\pi}_{j}). \]  
\end{theorem}

With Theorem \ref{thm:VI}, for each $j$, we can get a p-value for each permutation order $\pi$ with the test statistic $T^\pi_{j} = VI_{j}^\pi/\sigma^\pi_{j}$ under the hypothesis that $H_{0_{\pi}}: VI_{j}^{\pi}=0$ where $\pi\in \Pi(p)$, and we have \[p_{j}^{(max)} := \max_{\pi \in \Pi(p)} p_{j}^\pi <\alpha \Leftrightarrow Y \notindependent X_j | X_{-j}.\] Denote $p_{n,j}^{\pi_k}$ as estimated p-value for $p_{j}^\pi$ of feature $j$ under permutation $\pi_k \in \Pi$ with null hypothesis $H_{0_k}$, and $
p_{n,j}^{(1)} \leq ... \leq p_{n,j}^{(K)}
$
as the order statistics of $\{p^{\pi_k}_{n,j}: k=1,...,K\}$. From a multiple testing perspective, each permutation/ordering generates a feature importance statistic and a corresponding 
p-value for feature $j$; we then aggregate these p-values to make the best decision. MinShap modifies the Shapley value by taking the minimum marginal contribution across orderings, which targets the direct effect. The link between them is that MinShap operates on the same collection of ordering statistics, which we map to p-values via a normal approximation and perform multiple-testing procedures. Both MinShap and Max-p test the following hypothesis:

\[H_0: X_j \independent Y|X_{-j} \iff \exists\,k\in[K]\ \text{s.t.}\ VI_{j}^{\pi_k}=0
\iff \bigcup_{k=1}^K \{VI_{j}^{\pi_k}=0\};\]
\vspace{-0.1cm}
\[H_a: X_j \notindependent Y|X_{-j} \iff \forall\,k\in[K],\ VI_{j}^{\pi_k}\neq 0\iff \bigcap_{k=1}^K \{VI_{j}^{\pi_k}\neq 0\}.\]

\begin{corollary}\label{maxp_cor}
     If assumptions \ref{ass:faithful} and \ref{ass:revser_cau} hold in DAG modeling, then under the null hypothesis $H_0:X_j \independent Y|X_{-j}  $, the Max-p value $p_{n,j}^{(K)}$ estimated with sufficiently large $K = O(|S^*|)$ permutations is valid, ie: $\mathbb{P}(p_{n,j}^{(K)} < \alpha|H_0) < \alpha$, for all $\alpha \in [0,1]$.
\end{corollary}

The corresponding proof is available in Appendix \ref{app:cor}. Note that, in contrast to common multiple hypothesis testing - where the global null is \(H_0:\bigcap_{k=1}^{K}\{VI_{j}^{\pi_k}=0\}\) and is rejected if at least one \(VI_{j}^{(k)}\neq 0\) - our conjunction null is \(H_0:\bigcup_{k=1}^{K}\{VI_{j}^{\pi_k}=0\}\), which is rejected if \(VI_{j}^{\pi_k}\neq 0\) for all \(k\in[K]\).

Since a single p-value may be highly significant while all other p-values show no evidence for the alternatives, Max-p  might be conservative for testing. Further, the faithfulness assumption tends to often fail in the finite sample setting due to near cancellations~\cite{uhler_geometry_2013} which means we may need to allow some error.
Benjamini and Heller~\cite{benjamini_screening_2008} proposed partial conjunction hypothesis testing (PCHT), which is less stringent that at least $u$ out of $K$ tested hypotheses are false. When $u=1$, this is well-known global null hypothesis testing problem where we reject the null if there is at least one false null. Our Max-p value corresponds to the case when $u=K$, where we reject the null if all are false null. To define PCHT, let's consider null hypothesis $H_{0_1},...,H_{0_K}$ with corresponding p-value $p_{1}, p_2...,p_K$, then \[H_0^{u/K}: \text{at least $K-u+1$ nulls are true }\]\[ H_1^{u/K}: \text{at least $u$ out of $K$ are non-null}\]
where $u$ is a data–driven value selected on a tuning set on a range of $u$'s with the minimum loss/highest performance. To obtain a p-value for PCHT, we can choose any one of three adjusted p-value as proposed by Benjamini and Heller~\cite{benjamini_screening_2008} to provide a valid p-value, $p^{u/K}$, for $H_0^{u/K}$:
\begin{enumerate}
    \item General dependency p-value (Bonferroni):\\
    $p_{n,j}^{u/K} = (K-u+1)p_{n,j}^{(u)}$;
    \item Independent p-value (Stouffer):\\
    $p_{n,j}^{u/K} =2*( 1-\Phi(\frac{\sum_{k=1}^{K-u+1}z^{(k)}}{\sqrt{K-u+1}}))$;
    \item Independent p-value (Fisher):\\
    $p_{n,j}^{u/K} = \mathbb{P}(\chi^2_{2(K-u+1)}\geq -2\sum_{k=u}^{K}\log p_{n,j}^{(k)})$.
\end{enumerate}
Then apply the Holm adjustment to get $p_{n,j}^{*u/K}$ where $p_{n,j}^{*1/K} = p_{n,j}^{(1)},  
p_{n,j}^{*u/K} = \max\{p_{n,j}^{*\frac{u-1}{K}},\, p_{n,j}^{u/K}\},  u=2,...,K$. We have:
\begin{theorem}[Re-formulated from \cite{benjamini_screening_2008}]\label{thm:pcht}
 $p_{n,j}^{*u/K}$ is a valid p-value for the partial conjunction null hypothesis $H_0^{u/K}$, ie: $P(p_{n,j}^{*u/K} < \alpha|H_0^{u/K}) < \alpha$, for all $\alpha \in [0,1]$.
\end{theorem}
We remark that Bonferroni is more robust since it is valid under arbitrary dependence among p-values, whereas Stouffer and Fisher require independence. If there is any underlying structure of the selected permutation order set, then Bonferroni method is recommended. In our case, permutation sampled i.i.d. with replacement yields independent p-values across repetitions, so we can apply any of the three adjustment methods. The Bonferroni method is more conservative since it looks only at a single order statistic: the $(K-u+1)^{th}$ largest p-value, whereas Fisher and Stouffer methods aggregate all evidence across the tail of largest p-values. The procedure of computing adjusted p-values does not require any extra model retraining, so the computation complexity is the same as MinShap and Max-p. When screening for different levels of $u$, one might worry about multiple testing, Holm adjustment to all above three methods are used to control overall false discovery rate (FDR)~\cite{benjamini_screening_2008}. The procedure of the computation for Max-p and adjusted p-value is summarized in Algorithm \ref{alg:maxp}. Fig.\ref{fig:flowchart} provides an
overview of our procedure in this paper.
\begin{figure}[tb]
    \centering
\includegraphics[width=0.8\textwidth]{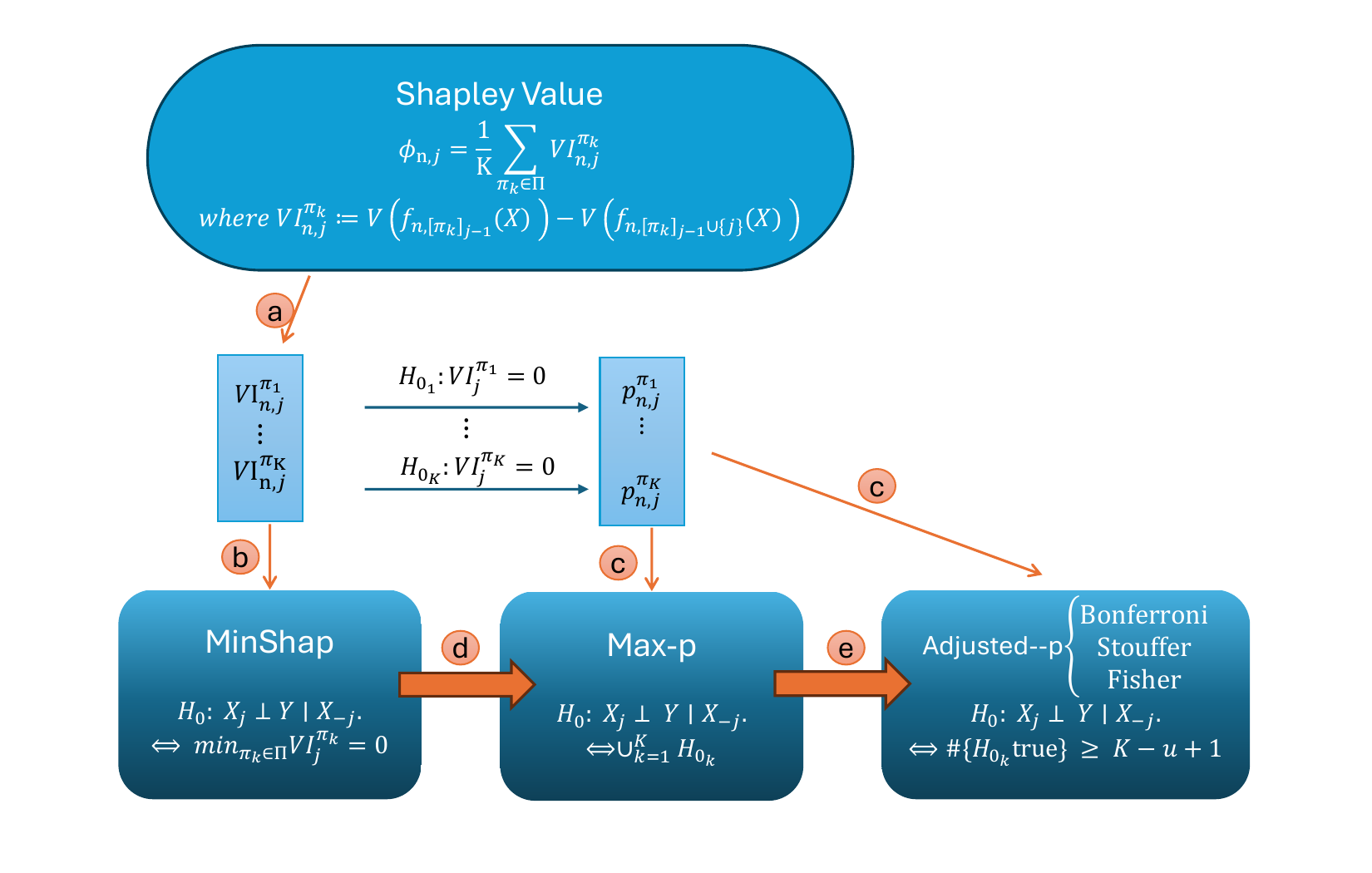}
    \caption{Overview of the MinShap algorithm and variations. \textcircled{a}: Shapley values are approximated by averaging marginal contributions/feature importance over $K$ randomly sampled permutations; \textcircled{b}: MinShap only tests the minimum contribution under faithfulness in DAGs; \textcircled{c}: Each feature importance test statistic yields a p-value test; \textcircled{d}: Connect MinShap to Max-p via multiple test perspective; \textcircled{e}: PCHT, a resolution for conservativeness of Max-p.}
    \label{fig:flowchart}
\end{figure}

\subsection{Practical guidelines for the choice of methods}
We now outline some practical guidance on which methods we could choose over others in different settings. When the Type II error is close to 0, MinShap and Max-p perform similarly to the adjusted p-value methods, and it is therefore sufficient to use MinShap and Max-p alone. A potential concern is the conservativeness of MinShap and Max-p, that is, the Type II error is not close to 0. In this case, one may instead apply the adjusted p-value methods, which can control the inflated Type II error without incurring a substantial additional computational cost.

In a data application, where the significant feature set is unknown, a practical strategy is to apply all methods initially. If their results are largely consistent, then MinShap and Max-p are preferred, since the choice 
$u=K$ in the adjusted p-value approach yields a procedure equivalent to them. If, however, the results differ substantially, this may indicate that the Type II error of MinShap and Max-p is non-negligible; in that case, the adjusted p-value methods can be applied.

\section{Empirical Evidence}\label{sec:sim}
In this section, we will present empirical evidence both through simulations and real data to validate our MinShap algorithm. We also compare them with other state-of-the-art methods such as GCM, LOCO and Lasso. We assess their accuracy and stability for the linear
model, non-linear non-additive model, conditional interaction model and logistic model. Specifically, we consider 20 covariates $X\sim \mathcal N(0,\Sigma_{20\times 20})$ with diagonal equals 1, and the following different settings are used:

\begin{enumerate}
    \item[(a)] Linear model:\\
    $Y \sim 4X_1 + 4X_2 + 3X_3*X_4 + 3X_5 + 2X_6 + 2X_5*X_6 + X_7 + X_8 +\epsilon$, where $ \text{Corr}(X_3,X_4) = 0.5$, $\epsilon \sim N (0, 1).$
            \item[(b)] Non-linear non-additive model  :\\
            $ Y \sim 2 \sin(X_1) + 2\log(|X_2| + 1) + X_1X_2 + 3 \cos(X_3+X_4) + \max(0, X_5) + X_6X_7X_8 +\epsilon$, where $ \Sigma \text{ is block diagonal matrix, and within-block correlations are }(0,0.2,0.5,0.8), \epsilon \sim N (0, 1)$;
        \item[(c)] Conditional interaction model :\\
        $Y = 1.5 X_1X_2 * I(X_3 > 0) +X_4 X_5 * I(X_3 < 0) +3X_6X_7 *I(X_8 > 0) +X_9X_{10} *I(X_8 < 0)+\epsilon$,  where $ \text{Corr}(X_1,X_2)=\text{Corr}(X_6,X_7)=0.9; \text{Corr}(X_4,X_5) = \text{Corr}(X_9,X_{10})=0.5, \epsilon \sim N (0, 1)$;
        \item[(d)] Logistic Model :\\
        $Y \sim f(2.5 X_1 +2.5X_2+2X_3X_4 +1.5X_5+1.5X_6+X^2_7+X^3_8)+\epsilon$, where $f$ is sigmoid function, $\text{Corr}(X_1,X_2) = 0.5, \epsilon \sim N (0, 0.1^2)$.
\end{enumerate}

\subsection{Feature selection comparison across methods}\label{sim:perf_comp}
 For all models, we use $3000$ samples over $100$ simulations. The number of permutations $K$ is $50$ for MinShap and Max-p value approaches. Our algorithms are model-agnostic: we evaluate them both using XGBoost and feed-forward neural networks for comparison and obtain consistent results. Specifically, we first mainly discuss results based on XGBoost, and for results in feed-forward neural networks, see Appendix \ref{app:sim_nn}. The overall performance comparison is presented 
in Fig.\ref{fig:all_perf_xgb} and the dashed line is the empirical threshold from random fair coin-flipping. From Fig.\ref{fig:all_perf_xgb}, we see that although Lasso requires the least training time, it only works relatively well in linear model and its accuracy are low in all other models. GCM tends to under-select significant features, resulting in low power, whereas LOCO tends to over-select insignificant features, leading to high Type I error. LOCO can be inefficient under strong feature dependence: correlated features can substitute for each other, inducing correlation bias~\cite{verdinelli_decorrelated_2024}. GCM alleviates the effect of feature correlation by considering the covariance of residuals from two regressions, but it suffers from a natural limitation: a covariance of zero does not necessarily imply independence. For example, the scenario occurs when feature $X_j$ follows a symmetric distribution. GCM and LOCO do not perform well when the model are not additive, as we can see that their corresponding F1 scores are low in Fig.\ref{fig:all_perf_xgb}. 
Across different models (a)-(d), the performance of MinShap and Max-p out-perform those three baseline methods with higher accuracy and F1 score, and maintain Type I error control. 

We assess stability of the selected features via the Jaccard index~\cite{mohana_survey_2016},  defined as the average pairwise similarity of selected feature sets:\[J_N
= \frac{2}{N(N-1)} \sum_{i=1}^{N-1} \sum_{j=i+1}^{N} J_{ij}.\]
 where $N$ is number of simulation, $J_{ij}\in [0,1]$ is the proportion of commonly selected features in two simulations out of the union of all features selected in two simulations. A higher Jaccard index indicates a higher overall similarity and lower dispersion among the selected feature sets, and thus the algorithm is more stable~\cite{mohammadi_robust_2016}. From Table~\ref{tb:jaccard_xgb}, we find that MinShap and Max-p are more stable than LOCO and GCM. While Lasso is stable, its accuracy is completely compromised in all nonlinear settings.
\begin{figure}[tb]
    \centering
    \includegraphics[width=0.8\textwidth]{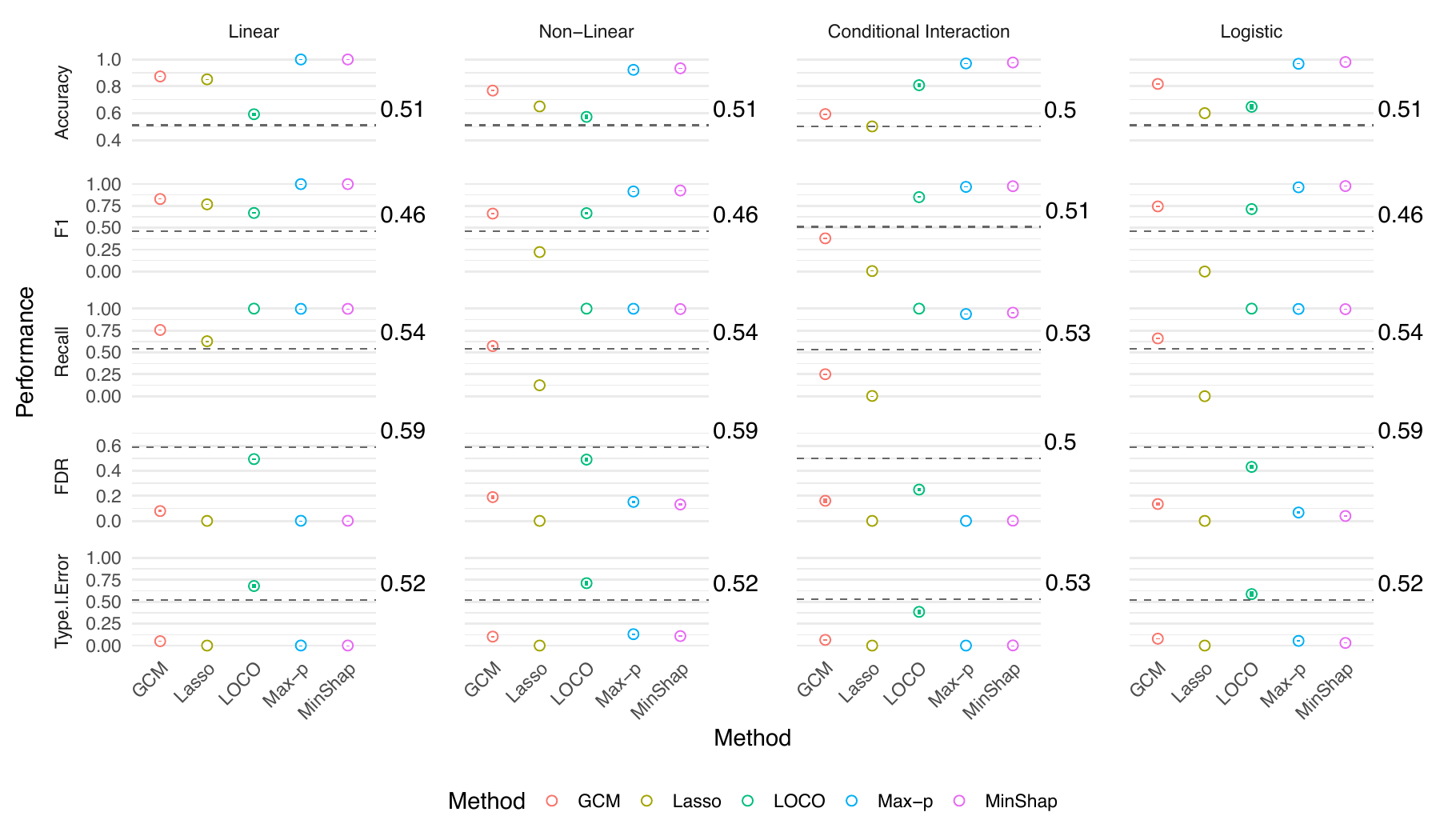}
    \caption{Average performance measure with standard deviation for model (a)-(d) for feature selection comparison using XGBoost. We see that MinShap and Max-p outperform other feature selection methods.}
    \label{fig:all_perf_xgb}
\end{figure}

\begin{table}[H]
\centering
\caption{Jaccard Index for the stability of feature selection methods on model (a)--(d) using XGBoost.}
\label{tb:jaccard_xgb}
    \begin{small}
\begin{tabular}{@{}lccccc@{}}
\toprule
& \multicolumn{5}{c}{Feature Selection Method} \\
\cmidrule(lr){2-6}
Model & MinShap & Max-p & LOCO & GCM & Lasso \\
\midrule
(a) & 0.98 & 0.99 & 0.71 & 0.84 & 0.82 \\
(b) & 0.78 & 0.77 & 0.74 & 0.51 & 1.00 \\
(c) & 0.91 & 0.89 & 0.71 & 0.54 & 0.94 \\
(d) & 0.98 & 0.87 & 0.67 & 0.72 & 1.00 \\
\bottomrule
\end{tabular}
\end{small}
 \vskip -0.1in 
\end{table}

\subsection{Comparison to stability selection}

One of the motivations of our MinShap and Max-p algorithms is that they increase stability. In this section, we will compare the performance in terms of various accuracy metrics as well as run-time with the stability selection approach developed in~\cite{meinshausen_stability_2010, shah_variable_2013}. Stability selection is a resampling-based feature selection method that repeatedly runs a base feature selection method on different sub-samples of the data and retains variables with high selection frequency across splits.
 Specifically, we will apply the stability selection to LOCO, GCM and Lasso to model (a)-(d) over 100 simulations with a $50\%$ sub-sampling rate and a selection threshold of $0.8$. Then, we compare their performance with MinShap and Max-p, which, as discussed in the previous section \ref{Sec_pvalue}, can be regarded as the multiple testing based feature selection method. The setup for MinShap and Max-p remain the same as previous section \ref{sim:perf_comp} with XGBoost method. Our results both shows how our method has similar performance to LOCO-stability but also significantly shorter run-time.
\begin{figure}[htbp]
    \centering
    \includegraphics[width=0.8\textwidth]{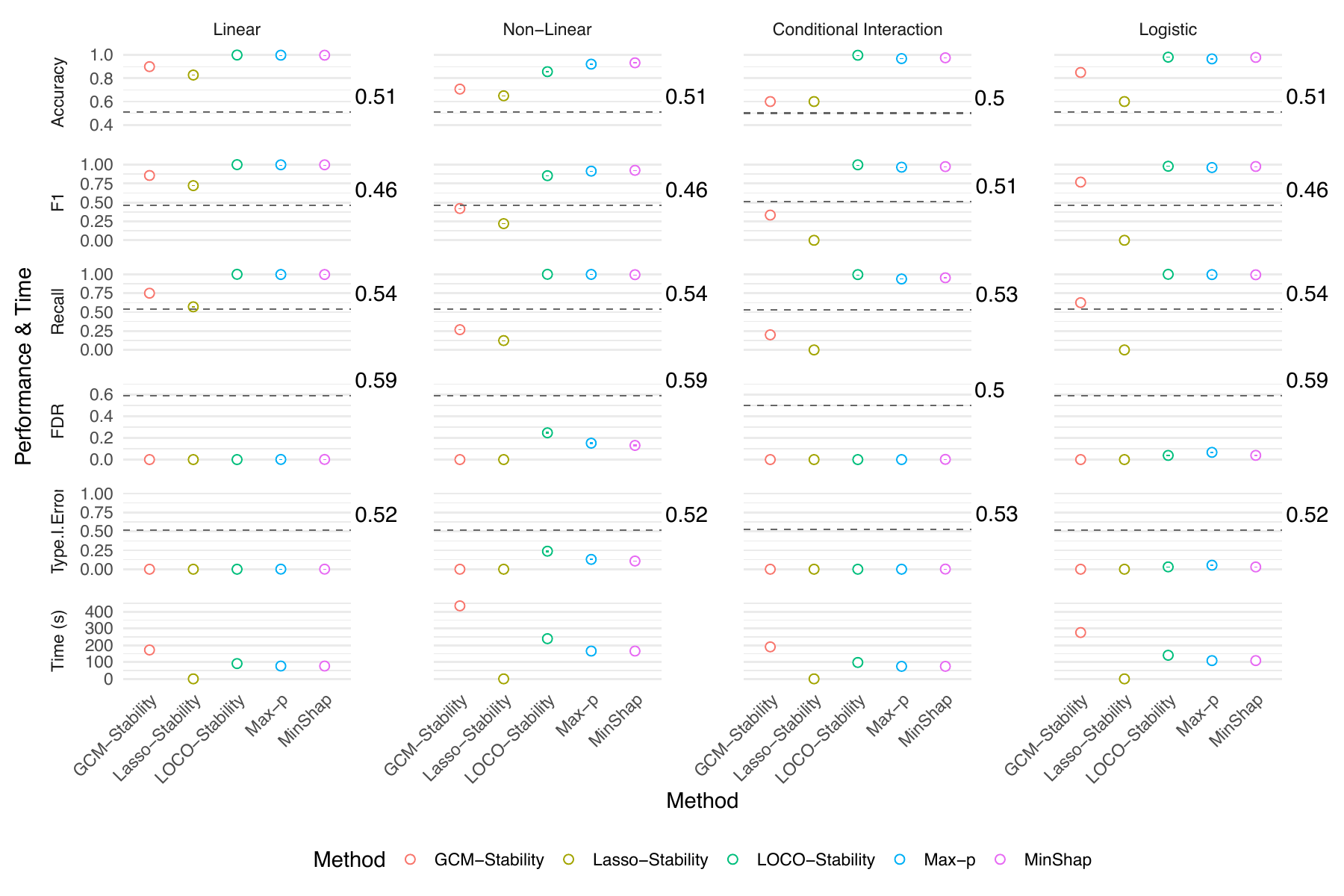}
    \caption{Average performance measure and time with standard deviation for model (a)-(d) for multiple testing based and stability selection based feature selection comparison using XGBoost. We see that MinShap and Max-p are more efficient than stability-based feature selection methods.}\label{fig:perf_xgb_with_stab}
\end{figure}

 From the Jaccard index in Table~\ref{tb:jaccard_xgb_with_stab}, we observe that the stability of our method and stability selection are similar especially for LOCO-stability. Fig.~\ref{fig:perf_xgb_with_stab} shows that after applying stability selection, the performance of Lasso and GCM remains roughly unchanged or is slightly improved in terms of accuracy and F1 score. The performance of LOCO improves significantly. Moreover, the Type I error and FDR are both well controlled for all three baseline feature selection methods after applying stability selection. However, despite the improvements brought by stability selection, our algorithms achieve similar performance in models (a), (c), and (d), and outperforms LOCO with stability selection in model (b) while always having significantly faster run-time. Therefore, MinShap and Max-p are more efficient.

\begin{table}[H]
\centering
\caption{Jaccard Index for the stability comparison of multiple testing based and stability selection based feature selection methods on model (a)--(d) using XGBoost.}
\label{tb:jaccard_xgb_with_stab}
    \begin{small}
\begin{tabular}{@{}lccccc@{}}
\toprule
& \multicolumn{5}{c}{Feature Selection Method} \\
\cmidrule(lr){2-6}
Model & MinShap & Max-p  & LOCO-Stability  & GCM-Stability  &Lasso-Stability \\
\midrule
(a) & 0.98 & 0.99 & 1.00  & 1.00  & 0.85\\
(b) & 0.78 & 0.77 & 0.72  & 0.92  &0.98\\
(c) & 0.91 & 0.89  & 0.99  & 1.00  &1.00\\
(d) & 0.98 & 0.87  & 0.93  & 1.00 &1.00 \\
\bottomrule
\end{tabular}
\end{small}
 \vskip -0.1in 
\end{table}

\subsection{Strength and limitation of MinShap/Max-p and the application of other adjusted p-value methods}
In this section, we will further discuss the strength and limitation of MinShap and Max-p, and provide practical guidance for when we use adjusted p-value methods. We will mainly focus on the first two models provided in section \ref{sec:sim}: linear and non-linear models. We will discuss the relationship among the number of permutation orders $K$, the proportion of significant features $|S^*|$, the sample size and further extend to high-dimensional case.
\begin{figure}[htbp]
    \centering
  \includegraphics[width=0.8\linewidth]{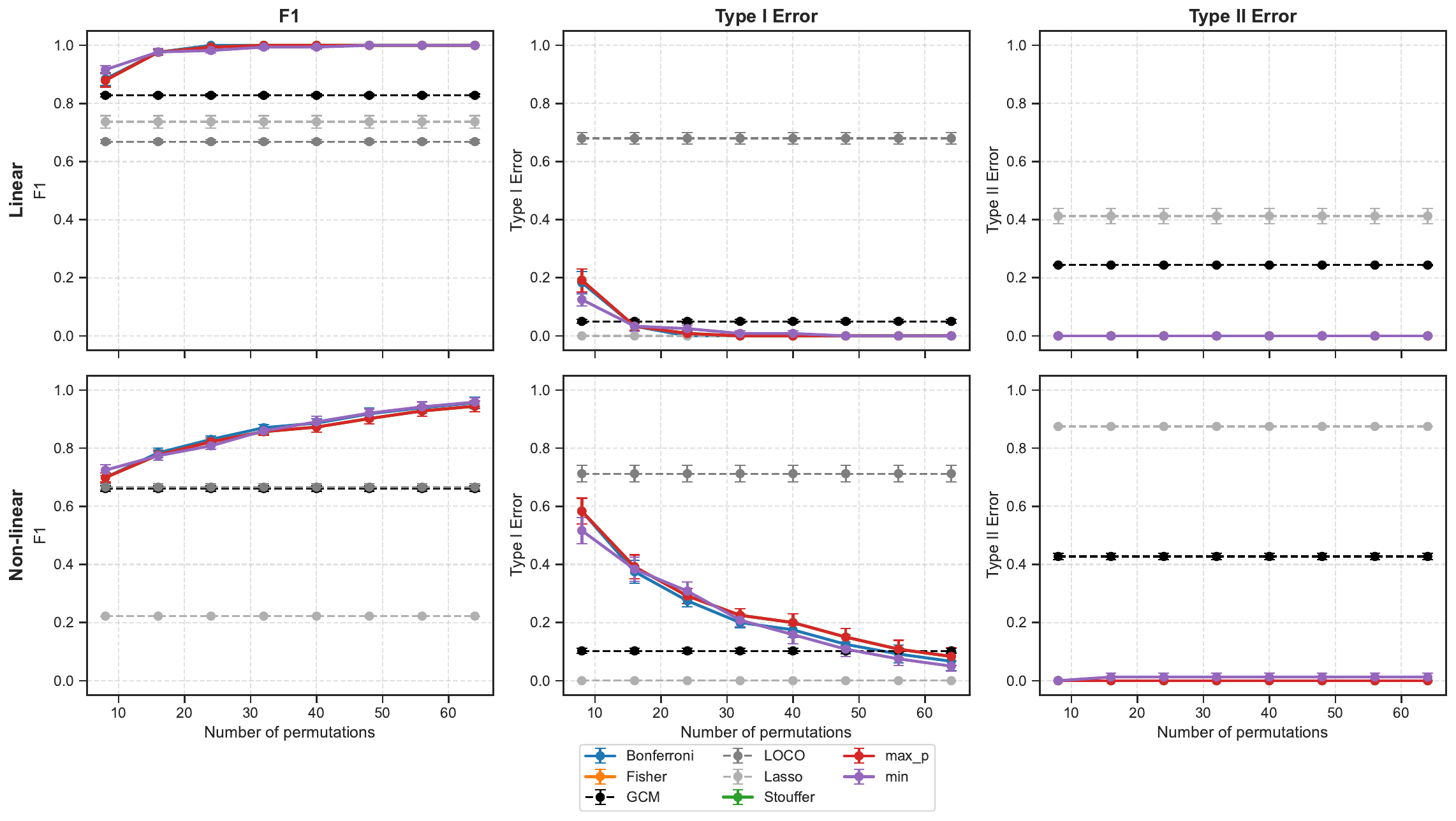}
    \caption{Average performance measure with standard deviation of different number of permutation for linear and non-linear model with sample size 3000. We see that when Type II error is nearly 0, more permutations drive higher performance of MinShap and Max-p.}
    \label{fig:perm_compare}
\end{figure}
\begin{figure}[htbp!]
    \centering
  \includegraphics[width=0.8\linewidth]{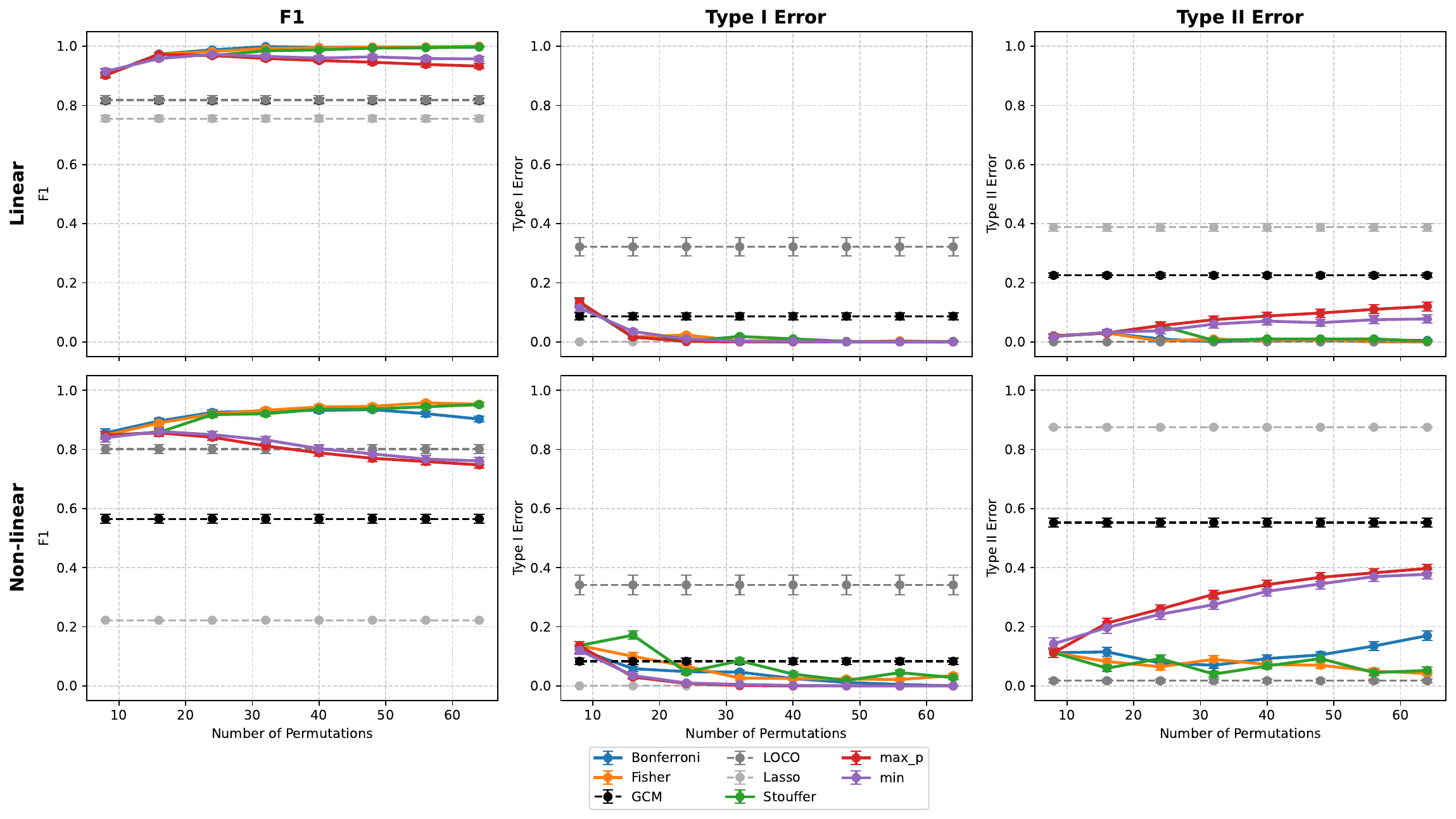}
    \caption{Average performance measure with standard deviation of different number of permutation for linear and non-linear model with sample size 1000. We see that when Type II error is between 0 and 1, more permutations drive higher Type II error of MinShap and Max-p.}
    \label{fig:perm_compare_1000}
\end{figure}

We first assess the effect of the number of permutation $K$ to MinShap and Max-p value methods. For Fig. \ref{fig:perm_compare}, the setup of the model is the same as section \ref{sec:sim}, except that we use different number permutations that is proportional to $|S^*|=8$, with $K$ range from 8 to 64. The screening range $ u \in [K-7,K]$ to avoid overlap with the previous selecting interval. When the sample size is large enough, and the Type II error is extremely close to 0, the more permutations we perform, the higher F1 score MinShap and Max-p value will achieve. The higher number of permutation $K$ drives the Type I error down. Next, we discuss the relation between the number of permutation $K$ and when Type II error is between 0 and 1. We still compare the same linear and non-linear model but with sample size 1000. The number of permutation $K$ still ranges from 8 to 64 with the screening range $ u \in [K-7,K]$ . From Fig. \ref{fig:perm_compare_1000}, we still see that more permutations always drive the Type I error down for all methods. When the Type II error is between 0 and 1, more permutations make the Type II error of MinShap and Max-p increase. Thus, when Type II error is not close to 1, the adjusted p-value methods can be applied.

\begin{figure}[htbp]
    \centering
    \includegraphics[width=0.8\linewidth]{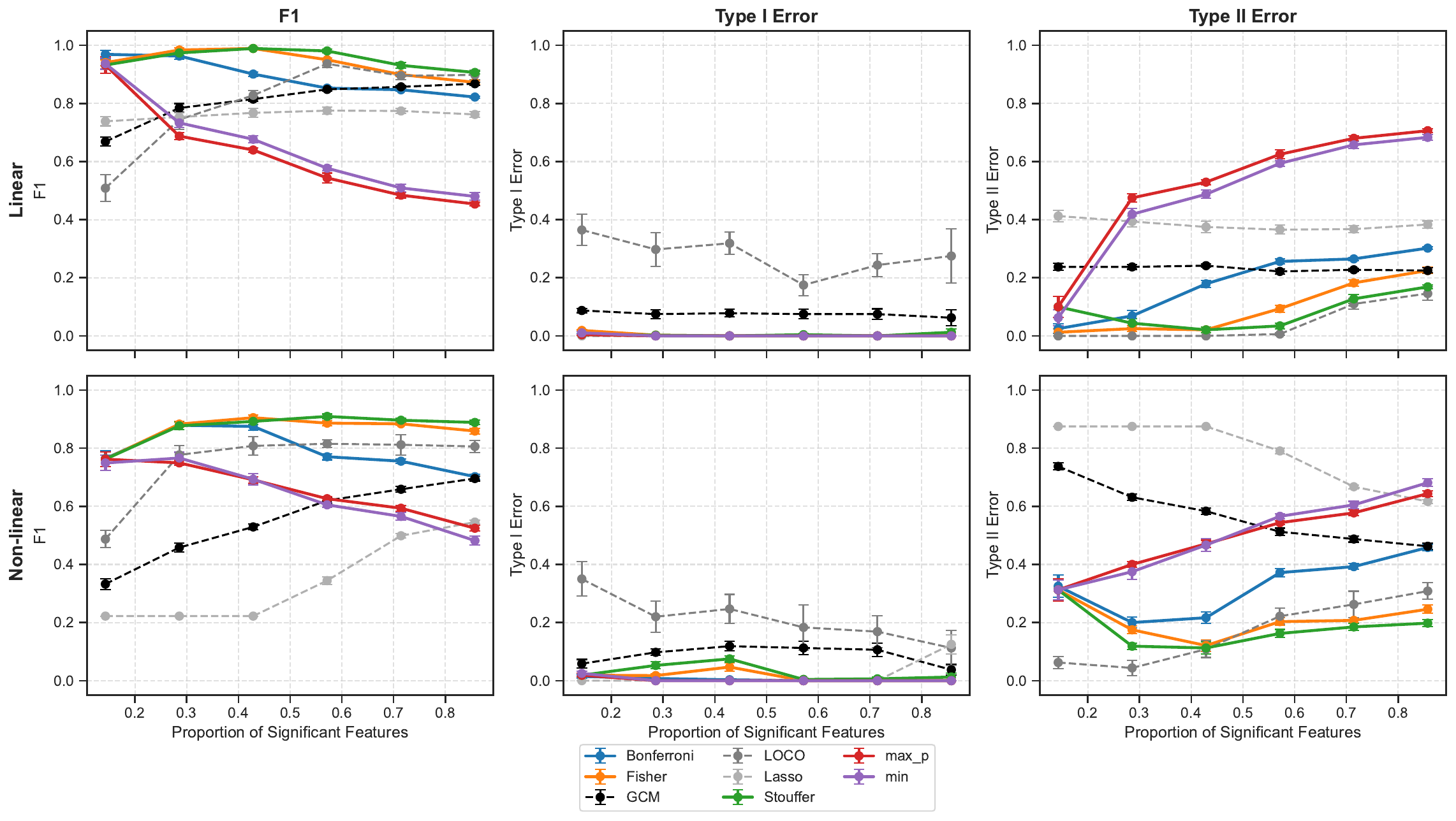}
    \caption{Average performance with standard deviation of varied proportion of significant features for linear and non-linear model. We see the advantage of the adjusted p-value tests in dense regime.}
    \label{fig:sig_prop_compare}
\end{figure}
\begin{figure}[htbp]
    \centering  \includegraphics[width=0.8\linewidth]
{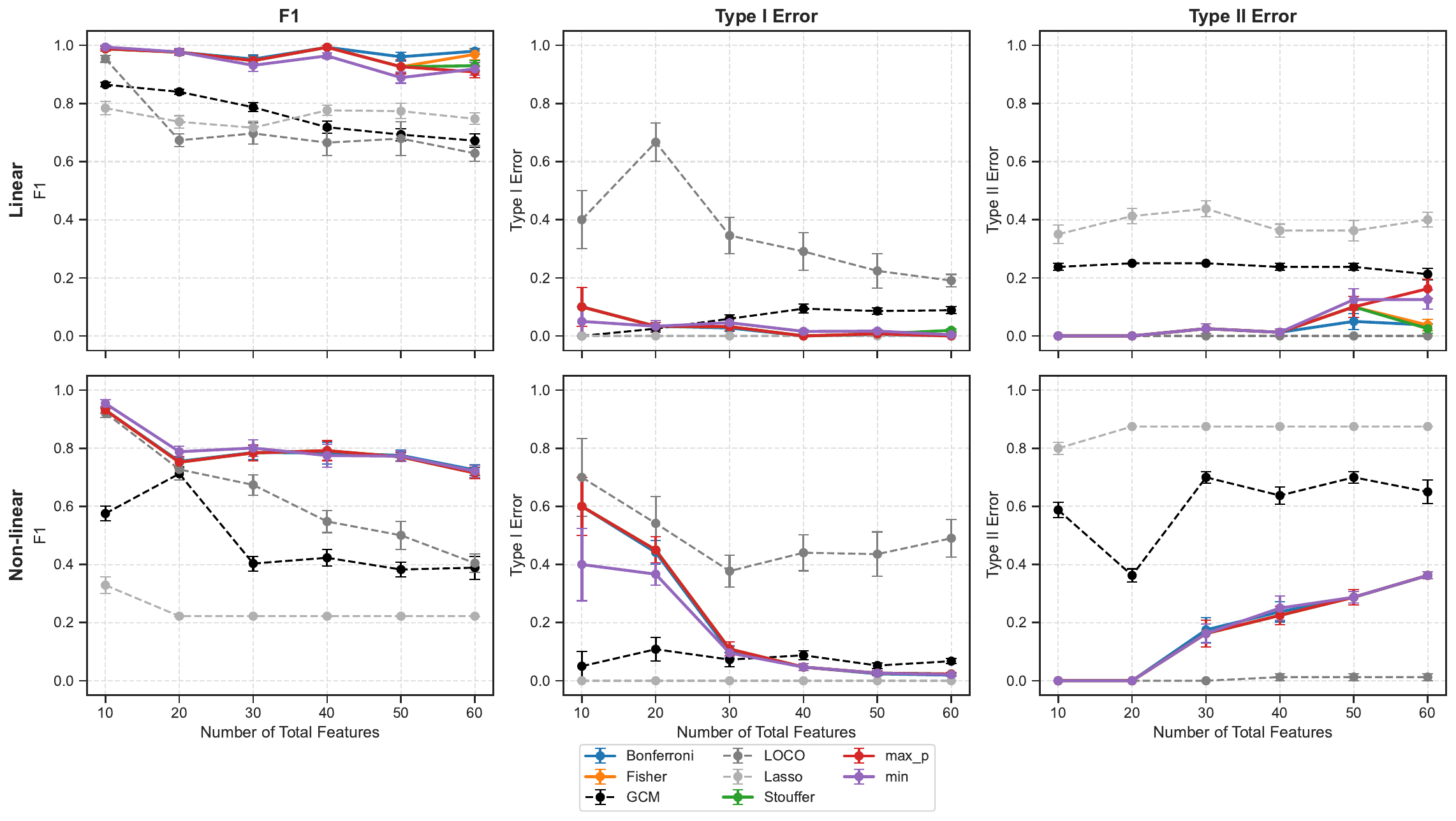}
    \caption{Average performance measure with standard deviation of fixed number of significant features with increasing number of total features for linear and non-linear model. We see that MinShap and Max-p work well under sparsity regime.}
    \label{fig:fix_s_compare}
\end{figure}

Next, we explore how the proportion of significant feature, $|S^*|$, affect the performance of MinShap and Max-p value methods. 
\begin{figure}[htbp!]
    \centering    \includegraphics[width=0.8\linewidth]{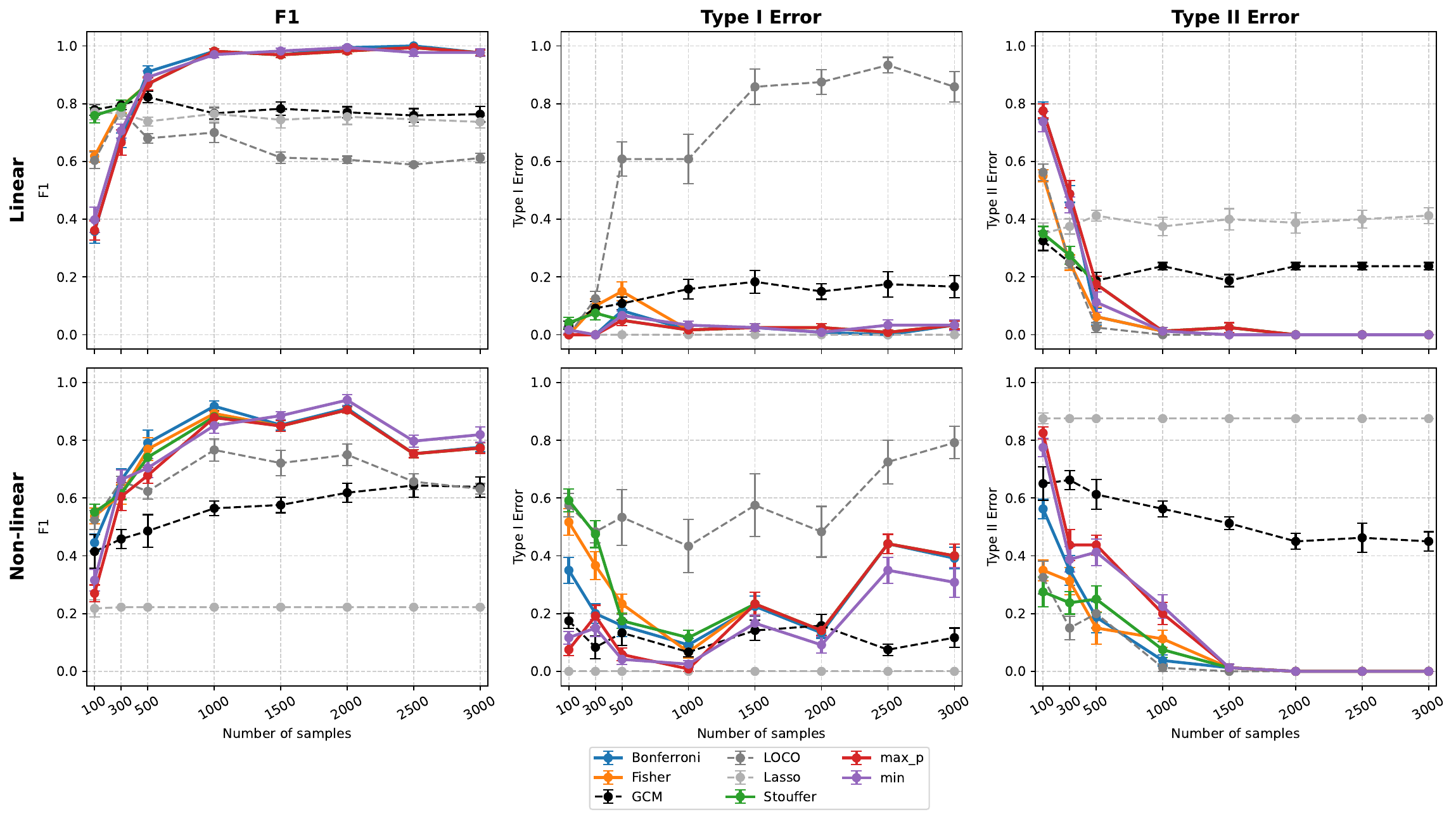}
    \caption{Average performance with standard deviation of different sample size for linear and non-linear model. We see the advantage of the suite of adjusted p-value (Bonferroni, Stouffer, Fisher) tests in lower-sample setting.}
    \label{fig:size_compare}
\end{figure}
We fix $p=56$, and vary the signal proportion by repeating the  same pattern of significant features; each step adds $8$  additional significant features by repeating the linear and nonlinear pattern. We set the number of permutations to $K =2|S^*|$. The best $u$ for adjusted p-value is chosen from  range  $[0.7K,K]$ with the highest F1 score. In Fig.\ref{fig:sig_prop_compare}, we find that the MinShap and Max-p value approach work well when the model is sparse.  As $|S^*|$ increases, the F1 score for both methods gradually decrease whereas F1 score of LOCO and GCM increase. When the sparsity structure of the model does not hold, the adjusted p-value methods can be used to alleviate the conservative effect of MinShap and Max-p caused by dense signal in the model. 

Slightly different from previous setup, we now fix the number of significant features at $|S^*| =8$, but increase the total number of features $p$ from 10 to 60, the number of permutation is $K=2|S^*|$, and the screening range $u \in [0.7K,K]$. From the Fig.\ref{fig:fix_s_compare}, in terms of both F1 score and Type I error, we see that MinShap and Max-p work well compared to other baseline feature selection methods when the sparsity holds, or equivalently, when the noise the in data is getting higher, and the adjusted p-value methods do not help with large improvement in this case.

Then, we compare how sample size $n$ affect their performance. We keep the models' pattern fixed with varied sample size  from $100$ to $3000$, $K =2|S^*|$, and the best $u$ is chosen from range $[0.7K,K]$. In Fig.\ref{fig:size_compare}, we find that when sample size is small, GCM and LOCO have higher F1 score than MinShap and Max-p, the Type II error of MinShap and Max-p is large and our adjustment methods such as Bonferroni, Stouffer and Fisher help reduce Type II error.

To summarize, the MinShap and Max-p coincide with adjusted p-value methods when the Type II error is close to 1. When the Type I error lies between 0 and 1, increasing the number of permutations decreases the power of our MinShap and Max-p. Also, MinShap and Max-p work well when the model is sparse and the sample size is relatively large. As the number features increase, the sample size decreases or the sparsity does not hold, the Bonferroni, Stouffer and Fisher adjusted p-value methods improve performance. In particular, Bonferroni is relatively more conservative, and Fisher and Stouffer yield comparable power.

\subsection{High-dimensional setting}

Next, we focus on high-dimensional setting. We set the total number of features $p=200$, and the pattern of significant features repeats twice (ie: $|S^*|=16$ in linear and non-linear models), the remaining $168$ variables are independent noise variables, the number of permutation $K=1.5|S^*|$, the screening level of $u$ in $[0.6K,K]$ and the sample size increases from 200 to 500. From Fig.\ref{fig:HD}, we see that Lasso works well in sparse linear model, but fails to keep its performance in non-linear model. In high-dimensional case, when the sample size is small, the performance of MinShap and Max-p is not ideal even if the Type I error is still under controlled. The Type II error of MinShap and Max-p is affected by the sample size. As sample size increases, the Type II error decreases. In linear model, the adjusted Stouffer p-value method performs slightly better than LOCO and GCM. In non-linear model, the Stouffer and Fisher adjusted p-value methods are consistently better than other baseline methods. Thus, in the high-dimensional setting where $p$ is relatively close to $n$, the adjusted p-value methods are favored.

\begin{figure}[htbp]
    \centering
\includegraphics[width=0.8\linewidth]{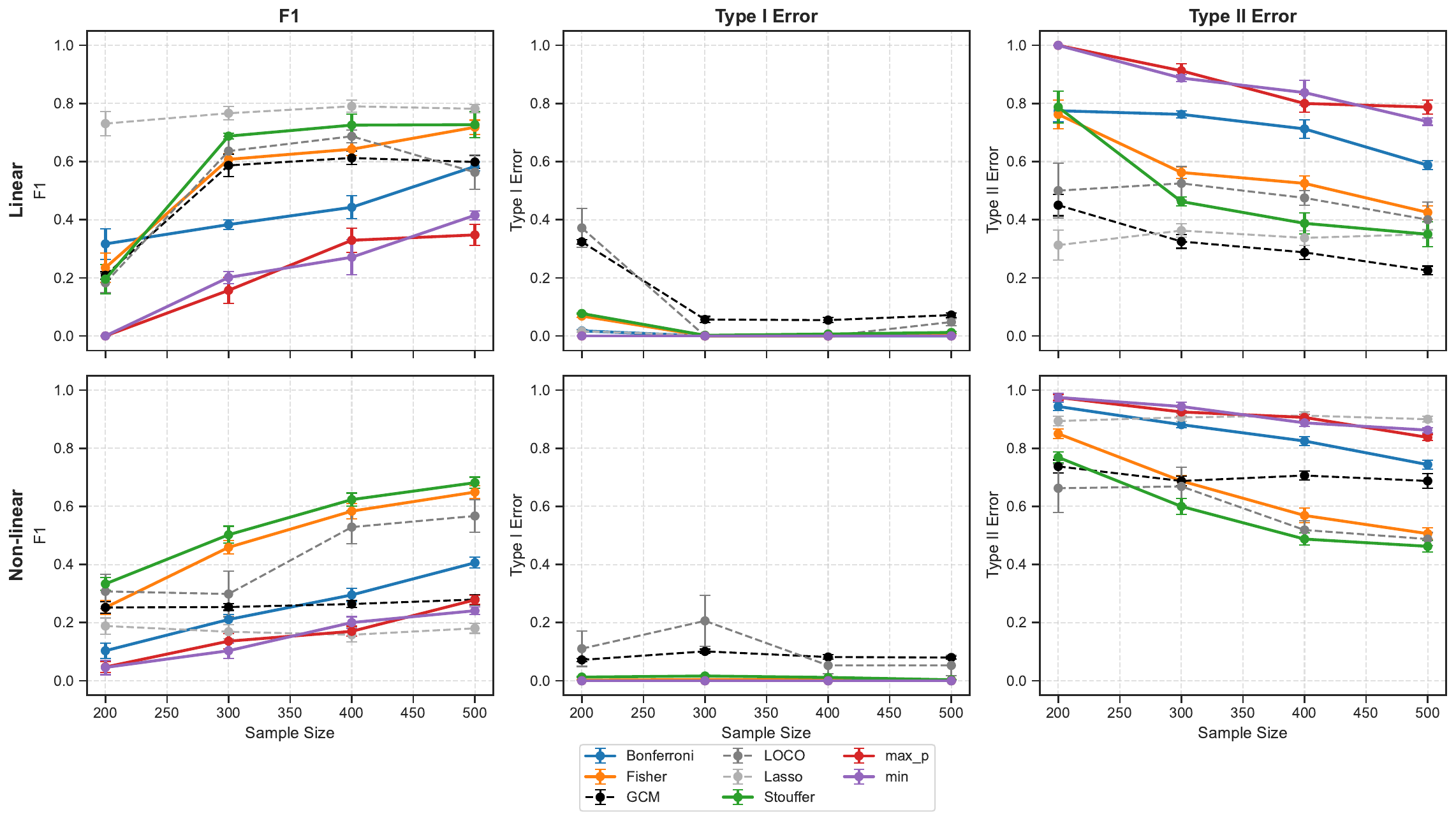}
    \caption{Average performance with standard deviation of high-dimensional setting for linear and non-linear model. We see the advantage
of the suite of adjusted p-value tests in in high-dimensional setting.}
    \label{fig:HD}
\end{figure}

\subsection{Real Data Analysis}
In this section, we apply the variable selection methods MinShap, Max-p, GCM and LOCO to two real and well-studied datasets, the well-known wine quality and California housing datasets. Each dataset is analyzed using 5-fold cross-validation. Variable selection is carried out within each training fold, and the mean squared prediction error on the corresponding test fold is used to evaluate predictive performance. In addition, the stability of the feature selections is measured by Jaccard index. Additional real-data experiment is provided in the Appendix \ref{app:real}.

\paragraph{Wine quality data:}
With increasing interest in wine, quality certification has become essential, relying on sensory evaluations by human experts. Our goal is to predict wine quality using objective analytical measurements collected during the certification process. The Wine Quality dataset, introduced by Cortez et al.~\cite{cortez_modeling_2009}, consists of $n = 1599$ red wine samples with $p = 11$ physicochemical features (e.g., acidity, alcohol content, etc.). The response is a wine quality score on a $0$--$10$ scale, reflecting the perceived quality of each wine. From Table \ref{wine_tab}, we find that features selected by LOCO are much less stable than other methods. The most
significant features frequently selected by both MinShap and Max-p using XGBoost and random forest(ie: "alcohol", "sulphates" and "volatile
acidity") suggest that physio-chemical properties play an important role in predicting wine quality.
These findings are consistent with prior literature on wine quality analysis~\cite{molnar_relating_2023}. GCM always select both significant and moderate significant features, such as "pH", "density", etc.~\cite{jain_machine_2023,tchakounte_wine_2024}.

\begin{table}[h]
\centering
\caption{Comparison of methods in terms of prediction error and selection stability on wine quality data under two models.}
\small
\renewcommand{\arraystretch}{1.2}
\begin{tabular}{llccc p{7cm}}
\hline
\multicolumn{1}{c}{Model} &
\multicolumn{1}{c}{Method} &
\multicolumn{1}{c}{MSE} &
\multicolumn{1}{c}{se(MSE)} &
\multicolumn{1}{c}{Jaccard index} &
\multicolumn{1}{c}{Selected Features} \\
\hline

\multirow{4}{*}{XGBoost}
& MinShap & 0.4045 & 0.0152  & 1.00 &
[`alcohol', `sulphates', `total sulfur dioxide', `volatile acidity'] \\
& Max-p   & 0.4048 & 0.01495 & 0.92 &
[`alcohol', `sulphates', `total sulfur dioxide', `volatile acidity'] \\
& LOCO    & 0.4027 & 0.01580 & 0.69 &
[`alcohol', `sulphates', `total sulfur dioxide', `volatile acidity'] \\
& GCM     & 0.4039 & 0.01548 & 0.83 &
[`alcohol', `density', `pH', `sulphates', `total sulfur dioxide', `volatile acidity'] \\

\hline

\multirow{4}{*}{\shortstack{Random\\forest}}
& MinShap & 0.4288 & 0.0156 & 1.00 &
[`alcohol', `sulphates', `volatile acidity'] \\
& Max-p   & 0.4288 & 0.0156 & 1.00 &
[`alcohol', `sulphates', `volatile acidity'] \\
& LOCO    & 0.4285 & 0.0156 & 0.69 &
[`alcohol', `sulphates', `volatile acidity'] \\
& GCM     & 0.4244 & 0.0163 & 0.83 &
[`alcohol', `chlorides', `density', `fixed acidity', `pH', `sulphates', `total sulfur dioxide', `volatile acidity'] \\
\hline
\end{tabular}
  \vskip -0.1in
\label{wine_tab}
\end{table}

\paragraph{California housing data:}
Accurate prediction of house prices requires accounting for many factors. We are interested in reducing pricing discrepancies by feature selection for predicting real estate prices. The California housing dataset, originally constructed by Pace and Barry (1997)~\cite{kelley_pace_sparse_1997} using data derived from the 1990 U.S. Census, contains $n = 20{,}640$ observations of California districts with $p = 8$ predictive features. These features describe demographic and housing characteristics, including median income (\texttt{MedInc}), median house age (\texttt{HouseAge}), average number of rooms and bedrooms (\texttt{AveRooms}, \texttt{AveBedrms}), population, average number of people per household in that district(\texttt{AveOccup}), and geographic location (\texttt{latitude} and \texttt{longitude}). The response variable is the median house value in each district in units of \$100,000. From Table \ref{tab:cal}, we see that both GCM and LOCO tend to select more moderate or less significant features. The feature selection performed by MinShap and Max-p are stable and always choosing the most significant features such as household income and the location of the property, which is aligned with prior studies
~\cite{capdevielle_feature_2025,chen_analysis_2023}.

\begin{table}[h]
\centering
\caption{Comparison of methods in terms of prediction error and selection stability on California housing data under two models.}
\small
\renewcommand{\arraystretch}{1.2}
\begin{tabular}{llccc p{7cm} }
\hline
\multicolumn{1}{c}{Model} &
\multicolumn{1}{c}{Method} &
\multicolumn{1}{c}{MSE} &
\multicolumn{1}{c}{se(MSE)} &
\multicolumn{1}{c}{Jaccard index} &
\multicolumn{1}{c}{Selected Features} \\
\hline

\multirow{4}{*}{XGBoost}
& MinShap & 0.3281 & 0.0079 & 0.92 &
[`AveOccup', `Latitude', `Longitude', `MedInc'] \\
& Max-p   & 0.3283 & 0.0080 & 0.82 &
[`AveOccup', `Latitude', `Longitude', `MedInc'] \\
& LOCO    & 0.3262 & 0.0083 & 0.77 &
[`AveOccup', `AveRooms', `HouseAge', `Latitude', `Longitude', `MedInc', `Population'] \\
& GCM     & 0.3257 & 0.0081 & 0.84 &
[`AveBedrms', `AveOccup', `HouseAge', `Latitude', `Longitude', `MedInc', `Population'] \\
\hline

\multirow{4}{*}{\shortstack{Random\\forest}}
& MinShap & 0.3456 & 0.0075 & 1.00 &
[`Latitude', `Longitude', `MedInc'] \\
& Max-p   & 0.3450 & 0.0072 & 0.90 &
[`Latitude', `Longitude', `MedInc'] \\
& LOCO    & 0.3204 & 0.0076 & 0.88 &
[`AveOccup', `AveBedrms', `HouseAge', `Latitude', `Longitude', `MedInc', `Population'] \\
& GCM     & 0.3162 & 0.0068 & 0.95 &
[`AveBedrms', `AveOccup', `HouseAge', `Latitude', `Longitude', `MedInc', `Population'] \\
\hline
\end{tabular}
\vskip -0.1in
\label{tab:cal}
\end{table}

\section{Discussion and Future Work}

In this paper, we adapt the Shapley value framework and show that, by replacing
the aggregation operator with the minimum, we obtain a feature selection
procedure under a null monotonicity condition. We discuss conditions under which
the null monotonicity condition is satisfied both in the statistical feature
selection setting and in mechanistic interpretability settings such as sparse
autoencoders and circuit-based explanations
\cite{elhage_mathematical_2021,bricken_towards_2023,cunningham_sparse_2023,wang_interpretability_2023}.

We then specialize to the model-agnostic statistical feature selection setting,
where we provide theoretical guarantees on Type I error for MinShap and show how
the framework can be combined with multiple testing procedures to provide
additional flexibility. Through simulations and real data examples, we
demonstrate that MinShap outperforms existing state-of-the-art methods such as
LOCO~\cite{lei_distribution-free_2018} and GCM~\cite{shah_hardness_2020} in terms of feature selection accuracy, F1 score, and stability. There are several promising directions for future research.

\paragraph{MinShap for mechanistic interpretability.}

The key assumption underpinning the MinShap algorithm is null monotonicity.
As discussed in Section~\ref{sec:minshap}, modern mechanistic interpretability
approaches, including sparse autoencoders
\cite{bricken_towards_2023,cunningham_sparse_2023}
and circuit-based analyses of neural networks
\cite{elhage_mathematical_2021,wang_interpretability_2023},
often approximately satisfy this assumption through additive or modular
representations. Consequently, the MinShap framework naturally extends beyond
statistical feature selection. Based on our theoretical analysis in the
statistical setting, one may expect MinShap to provide more stable identification
of redundant explanatory units than existing approaches. Establishing analogous
Type I error guarantees, deriving stability results, and evaluating MinShap on
modern large-scale neural network models remain interesting directions for
future work.

\paragraph{Algorithms for discovering minimal redundant subsets.}

As discussed in Section~\ref{sec:minshap}, exactly solving the MinShap
optimization problem requires minimizing over either all subsets or all
permutations, which is computationally intractable in general. The algorithm
considered in this paper approximates the minimum using randomly sampled
permutations. However, the null monotonicity property suggests that more
efficient optimization strategies may be possible. In particular, greedy
algorithms for submodular optimization~\cite{nemhauser_analysis_1978,krause_submodular_2014}
and branch-and-bound methods~\cite{lawler_branch-and-bound_1966} are natural
candidates for exploiting the upward-closed structure induced by null
monotonicity. Developing efficient optimization algorithms with theoretical
approximation guarantees remains an interesting direction for future work.

\paragraph{Other settings satisfying null monotonicity.}

Although the primary focus of this paper is statistical feature selection, and
to a lesser extent mechanistic interpretability, the null monotonicity principle
may arise in a broader range of interpretability and machine learning problems.
Potential examples include concept-based interpretability methods such as
TCAV~\cite{kim_interpretability_2018}, symbolic regression
\cite{koza_genetic_1992,udrescu_ai_2020}, and other structured representation
learning frameworks in which explanatory units admit stable hierarchical or
modular decompositions. Identifying additional settings in which null
monotonicity holds, together with studying whether MinShap provides reliable and
stable feature selection in these domains, remains an interesting direction for
future research.

\bibliographystyle{plain}
\bibliography{Shapley2.bib}
\newpage
\begin{appendices}
\section{Appendix - Theorem Related Proofs}

\subsection{Proof of Lemma~\ref{lem:predictive-sufficiency}}\label{pf_lemma_Predictive_sufficiency}

Since $\ell$ is a strictly proper scoring rule, the Bayes-optimal predictor
is the conditional distribution
\[
f_S^*(X_S)=P(Y\mid X_S).
\]
Moreover, the optimal expected score can be written as
\[
V(S)
=
C-\mathcal U(Y\mid X_S),
\]
where $C$ is a constant independent of $S$ and
$\mathcal U(\cdot)$ is the uncertainty functional induced by the loss.
For example,
\[
\mathcal U(Y\mid X_S)
=
H(Y\mid X_S)
\]
for logarithmic loss and
\[
\mathcal U(Y\mid X_S)
=
\mathbb E[\operatorname{Var}(Y\mid X_S)]
\]
for squared loss.

Suppose first that
\[
Y\independent X_{-S}\mid X_S.
\]
Then
\[
P(Y\mid X_S,X_{-S})
=
P(Y\mid X_S),
\]
so conditioning on the additional variables does not reduce uncertainty:
\[
\mathcal U(Y\mid X_S)
=
\mathcal U(Y\mid X).
\]
Hence
\[
V(S)
=
C-\mathcal U(Y\mid X_S)
=
C-\mathcal U(Y\mid X)
=
V([p]).
\]

Conversely, suppose that
\[
V(S)=V([p]).
\]
Since
\[
V(S)=C-\mathcal U(Y\mid X_S),
\]
we obtain
\[
\mathcal U(Y\mid X_S)
=
\mathcal U(Y\mid X).
\]
Because proper losses are strictly information-monotone, equality occurs
only when
\[
P(Y\mid X_S,X_{-S})
=
P(Y\mid X_S)
\quad\text{almost surely}.
\]
This is precisely the conditional independence relation
\[
Y\independent X_{-S}\mid X_S.
\]

Therefore,
\[
V(S)=V([p])
\iff
Y\independent X_{-S}\mid X_S,
\]
and the two minimum-cardinality optimization problems have identical
solution sets.

\subsection{Proof of Theorem \ref{main_thm}}\label{app:thm}
\begin{proof}
By Theorem \ref{thm:ordering}, we know that under the faithfulness assumption and there is no data leakage assumption, we are performing the following hypothesis testing:
    \[H_{0}: \min_{\pi \in \Pi(p)} VI_{0,j}^{\pi} = 0 \Leftrightarrow Y \independent X_j | X_{-j}\;;\; H_{a}: \min_{\pi \in \Pi(p)} VI_{0,j}^{\pi} >  0 \Leftrightarrow Y \notindependent X_j |X_{-j}\] 

Let's randomly select $K$ permutations, and recall that we define $\Pi = (\pi_1, \pi_2,...,\pi_K)$, and $VI_{n,j}^{\pi_k}, k \in [K]$ as the estimated variable importance of feature $j$ under permutation $\pi_k$. Define $t_j$ as the thresholds we use to decide if we  reject the null hypothesis. Define an event $A=\{\text{For a given order of permutation $\pi_k$: $VI_{n,j}^{\pi_k} \leq t_j $}\}$, and event $\widetilde{A}$ =\{Among the $K$ randomly selected order of permutations, there exists at least one order $\pi_k, k \in [K]$ s.t. $VI_{n,j}^{\pi_k} \leq t_j$\}.

Suppose we have $p$ features, and there are $|S_0^*|$ significant features.
Let $\widetilde{B} = \{B_1,B_2,...,B_I\}$ be the collection of minimal feature subset of feature $j$, where each $B_i \subset [p]\backslash \{j\}$ s.t. $V(B_i \cup \{j\}) -V(B_i) = 0$. Then, define the event $E(B_i) = \{\text{feature $j$ is placed after all features in $B_i$}\}$, we have
\[\mathbb{P}(E(B_i)) = \frac{\binom{p}{|B_i|+1}\cdot |B_i|!}{\binom{p}{|B_i|+1}\cdot |B_i+1|!} = \frac{1}{|B_i|+1}\]
\[\mathbb{P}(A) = \mathbb{P}(E(\widetilde{B})) = \mathbb{P}(\bigcup_{i=1}^I E(B_i))\geq \max_{B_i \in \tilde{B}} \mathbb{P}(E(B_i)) = \frac{1}{\min_{B_i \in \tilde{B}} |B_i| + 1} \geq \frac{1}{|S_0^*|+1}\]
Then, the probability that at least one out of K independent permutations yielding $VI_{n,j}^{\pi_k} = 0, k \in [K]$ is:
\[\mathbb{P}(\widetilde{A}) = 1-(1-\mathbb{P}(A))^K \geq 1-(1-\frac{1}{|S_0^*|+1})^K\]
Thus, for any fixed $\epsilon \in (0,1)$, if $K$ is sufficiently large with order $K =O( |S_0^*| )$ then we have $\mathbb{P}(\widetilde{A}) > 1-\epsilon$.
Finally, let's justify that the test statistic $VI_{n,j}^{(1)}$ is under Type I error control:
\begin{equation*}
    \begin{split}
        \mathbb{P}(VI_{n,j}^{(1)} >t_j|H_0) &= \mathbb{P}(VI_{n,j}^{(1)} >t|H_0, \widetilde{A}) \cdot \mathbb{P}(\widetilde{A})\\  
        &\leq \exp (\frac{1}{2} \lambda^2 \sigma_{n,j}^{2(1)} -\lambda t) \text{, by Chernoff Inq}
    \end{split}
\end{equation*}
To give a tight bound, we optimize $\lambda$ by minimizing the quadratic  equation, and we get $\hat{\lambda} = \frac{t}{\sigma^{2,(1)}_{n,j}}$. Then, we have 
\begin{equation*}
    \begin{split}
        \mathbb{P}(VI_{n,j}^{(1)} >t_j|H_0) &\leq \exp\big(\frac{1}{2} \frac{t^2}{\sigma_{n,j}^{2,(1)}}-\frac{t^2}{\sigma_{n,j}^{2,(1)}}\big)\\
        & =  \exp\big( -\frac{1}{2} \frac{t^2}{\sigma_{n,j}^{2,(1)}}\big)< \alpha
    \end{split}
\end{equation*}
Thus, we have threshold $t_j>\sqrt{-2\ln(\alpha)\sigma_{n,j}^{2,(1)}}$.
\end{proof}
\subsection{Type II error control}\label{app:type2}
To provide some further intuition of maximizing power, let's define an event $C=\{\text{For a given order}$ $\text{ of permutation $\pi_k$: $VI_{n,j}^{\pi_k} > t_j$}\}$.
\begin{equation*}
\begin{split}
        \text{Type II error} = 1-\mathbb{P}(VI_{n,j}^{(1)} >t_j|H_a)&= 1-\mathbb{P}(\bigcap_{k=1}^K VI_{n,j}^{\pi_k} > t|H_a)\\
        &=1-\prod_{k=1}^K \mathbb{P}(VI_{n,j}^{\pi_k} > t |H_a) \\
        & = 1-\mathbb{P}(C|H_a)^K
\end{split}
\end{equation*}
This result suggests that more permutations drive the Type II error down when $P(C|H_a)$ is close to 1. If $P(C|H_a) \in (0,1)$, fewer permutations $K$ yields better Type II error control. 

\subsection{Proof of Corollary \ref{maxp_cor}}\label{app:cor}
The MinShap uses concentration inequality to generate threshold whereas Max-p uses normal approximation and multiple hypothesis testing.
\begin{proof}
Under the union of $K$ null hypothesis, there is at least one true null. The proof for the equivalence that Max-p is also under Type I error control is easy to see, since 
\begin{equation*}
    \begin{split}
        \mathbb{P}(p_{n,j}^{(K)} < \alpha|H_0) &=\mathbb{P}(|Z_{n,j}^{(1)} |> t)  \\
        &\leq \mathbb{P}(|Z_{n,j}^{(k)}|>t)= \alpha
    \end{split}
\end{equation*}
where $Z_{n,j}^{(k)} = \frac{VI_{n,j}^{(k)}}{\sigma_{n,j}^{(k)}}$, $k \in [K]$ and $t = \Phi^{-1}(1-\alpha/2)$.
\end{proof}

\section{Appendix - Algorithms}\label{app:alg}
This section includes two algorithms: the first one is for Monte carol permutation sampling approximation of Shapley value based on feature importance context (see Algorithm \ref{alg:meanshap}), and the second one is the Max-p related p-value algorithm (see Algorithm \ref{alg:maxp}).

\begin{algorithm}[H]
\caption{Monte Carlo Permutation Sampling Approximation of Variable-Importance Shapley values}
\label{alg:meanshap}
\begin{algorithmic}[1]
\Require Dataset $(X,Y)$; measure $V$; number of features $p$; number of permutations $K$

\State Fit null model $f_{n,\emptyset}$; set $V_{\emptyset}\leftarrow V(f_{n,\emptyset})$
\State Initialize $\phi_{n,j}\leftarrow 0$ for all $j\in[p]$

\For{$k=1,\ldots,K$}
  \State Sample a random permutation $\pi_k$
  \State Set $V_{\mathrm{cur}} \leftarrow V_{\emptyset}$
  \For{$j=1,\ldots,p$}
    \State Set predecessor set: $\mathcal{P}^{\pi_k}_{j-1} \leftarrow \{\, i\in [p]: \pi_k(i) < \pi_k(j)\,\}$
    \State Fit model $f_{n,\mathcal{P}^{\pi_k}_{j-1}\cup\{j\}}$ and compute $V_{\mathrm{new}}\leftarrow V\!\big(f_{n,\mathcal{P}^{\pi_k}_{j-1}\cup\{j\}}\big)$
    \State $VI_{n,j}^{\pi_k}\leftarrow V_{\mathrm{new}} - V_{\mathrm{cur}}$
    \State $\phi_{n,j}\leftarrow \phi_{n,j}+VI_{n,j}^{\pi_k}$
    \State $V_{\mathrm{cur}}\leftarrow V_{\mathrm{new}}$
  \EndFor
\EndFor

\State \textbf{Return} $\phi_{n,j}\leftarrow \phi_{n,j}/K$ for all $j\in[p]$
\end{algorithmic}
\end{algorithm}

\begin{algorithm}[tb]
\caption{Max-$p$ and adjusted $p$-value testing for feature selection}
\label{alg:maxp}
\begin{algorithmic}[1]
\Require Dataset $(X,Y)$; measure $V$; features $p$;  number of permutations $K$; significance level $\alpha$
\Ensure Adjusted $p$-values $\tilde p^{\mathrm{Bon}}_u,\tilde p^{\mathrm{Fisher}}_u,\tilde p^{\mathrm{Stou}}_u$ and $p^{\mathrm{Max}}$ for each feature $j$

\State Fit null model $f_{n,\emptyset}$; set $V_{\emptyset}\leftarrow V(f_{n,\emptyset})$ and residuals $\mathbf e^2_{\emptyset}$
\State Allocate arrays $\phi^{\pi_k}_{n,j}$ and $\sigma^{2,\pi_k}_{n,j}$ for $j\in[p],\,k\in[K]$

\For{$k=1,\ldots,K$}
  \State Sample permutation $\pi_k$
  \State Set $V_{\mathrm{cur}}\leftarrow V_{\emptyset}$ and $\mathbf e^2_{\mathrm{cur}}\leftarrow \mathbf e^2_{\emptyset}$
  \For{$j=1,\ldots,p$}
    \State $\mathcal P^{\pi_k}_{j-1}\leftarrow\{\, i\in[p]: \pi_k(i)<\pi_k(j)\,\}$
    \State Fit $f_{n,\mathcal P^{\pi_k}_{j-1}\cup\{j\}}$; compute $V_{\mathrm{new}}$ and residuals $\mathbf e^2_{\mathrm{new}}$
    \State $\phi^{\pi_k}_{n,j}\leftarrow V_{\mathrm{new}}-V_{\mathrm{cur}}$; $\sigma^{2,\pi_k}_{n,j}\leftarrow Var(\mathbf e^2_{\mathrm{new}}-\mathbf e^2_{\mathrm{cur}})/n$
    \State $V_{\mathrm{cur}}\leftarrow V_{\mathrm{new}}$; $\mathbf e^2_{\mathrm{cur}}\leftarrow \mathbf e^2_{\mathrm{new}}$
  \EndFor
\EndFor

\For{$j=1,\ldots,p$}
  \State Convert $\{\phi^{\pi_k}_{n,j},\sigma^{2,\pi_k}_{n,j}\}_{k=1}^K$ to $z$-scores and $p$-values $\{p_{n,j}^{\pi_k}\}_{k=1}^K$
  \State Sort $p_{n,j}^{(1)}\le\cdots\le p_{n,j}^{(K)}$ and $|z|_{n,j}^{(1)}\le\cdots\le |z|_{n,j}^{(K)}$
  \State \textbf{Bonferroni-Holm:} for $u=1,\ldots,K$, set $p^{\mathrm{Bon}}_u\leftarrow (K-u+1)p_{n,j}^{(u)}$
  \State $\tilde p^{\mathrm{Bon}}_1\leftarrow p^{\mathrm{Bon}}_1$\; ; \; for $u=2,\ldots,K$ set $\tilde p^{\mathrm{Bon}}_u\leftarrow \min\big(\max(\tilde p^{\mathrm{Bon}}_{u-1},p^{\mathrm{Bon}}_u),1\big)$
  \State \textbf{Fisher-Holm:} for $u=1,\ldots,K$, compute $T^{\mathrm{Fisher}}_u\leftarrow -2\sum_{k=u}^K \log p_{n,j}^{(k)}$ and $p^{\mathrm{Fisher}}_u\leftarrow 1-F_{\chi^2_{2(K-u+1)}}(T^{\mathrm{Fisher}}_u)$
  \State $\tilde p^{\mathrm{Fisher}}_1\leftarrow p^{\mathrm{Fisher}}_1$ \; ; \; for $u=2,\ldots,K$ set $\tilde p^{\mathrm{Fisher}}_u\leftarrow \max(\tilde p^{\mathrm{Fisher}}_{u-1},p^{\mathrm{Fisher}}_u)$
  \State \textbf{Stouffer-Holm:} for $u=1,\ldots,K$, set $z^{\mathrm{Stou}}_u\leftarrow \frac{\sum_{k=1}^{K-u+1}|z|_{n,j}^{(k)}}{\sqrt{K-u+1}}$ and $p^{\mathrm{Stou}}_u\leftarrow 2\big(1-\Phi(|z^{\mathrm{Stou}}_u|)\big)$
  \State $\tilde p^{\mathrm{Stou}}_1\leftarrow p^{\mathrm{Stou}}_1$ \; ; \; for $u=2,\ldots,K$ set $\tilde p^{\mathrm{Stou}}_u\leftarrow \max(\tilde p^{\mathrm{Stou}}_{u-1},p^{\mathrm{Stou}}_u)$
  \State \textbf{Max-$p$:} $p^{\mathrm{Max}}\leftarrow p_{n,j}^{(K)}$
  \State Decision: reject $H_{0,j}$ if the selected $p$-value $<\alpha$
\EndFor
\end{algorithmic}
\end{algorithm}

\section{Appendix - Simulation}
\subsection{Experiment on neural network}\label{app:sim_nn}

Our modified Shapley value based variable selection procedure is model-agnostic and thus not restricted to tree-based algorithm. In addition to XGBoost, we also evaluate the method using a two-layer ReLU network as the predictive model for comparison. Since neural network training is more computationally intensive, we run these experiments at a reduced scale, decreasing the feature dimension from $p=20$ to $p=10$ with only the first 5 features significant in all 4 models, and the sample size from $n=3000$ to $n=1000$. The number of simulation is 50 and the number of permutation $K$ is 10 for MinShap and Max-p value approaches. 
\begin{figure}[tb]
\vskip 0.1in
    \centering
    \includegraphics[width=0.8\linewidth]{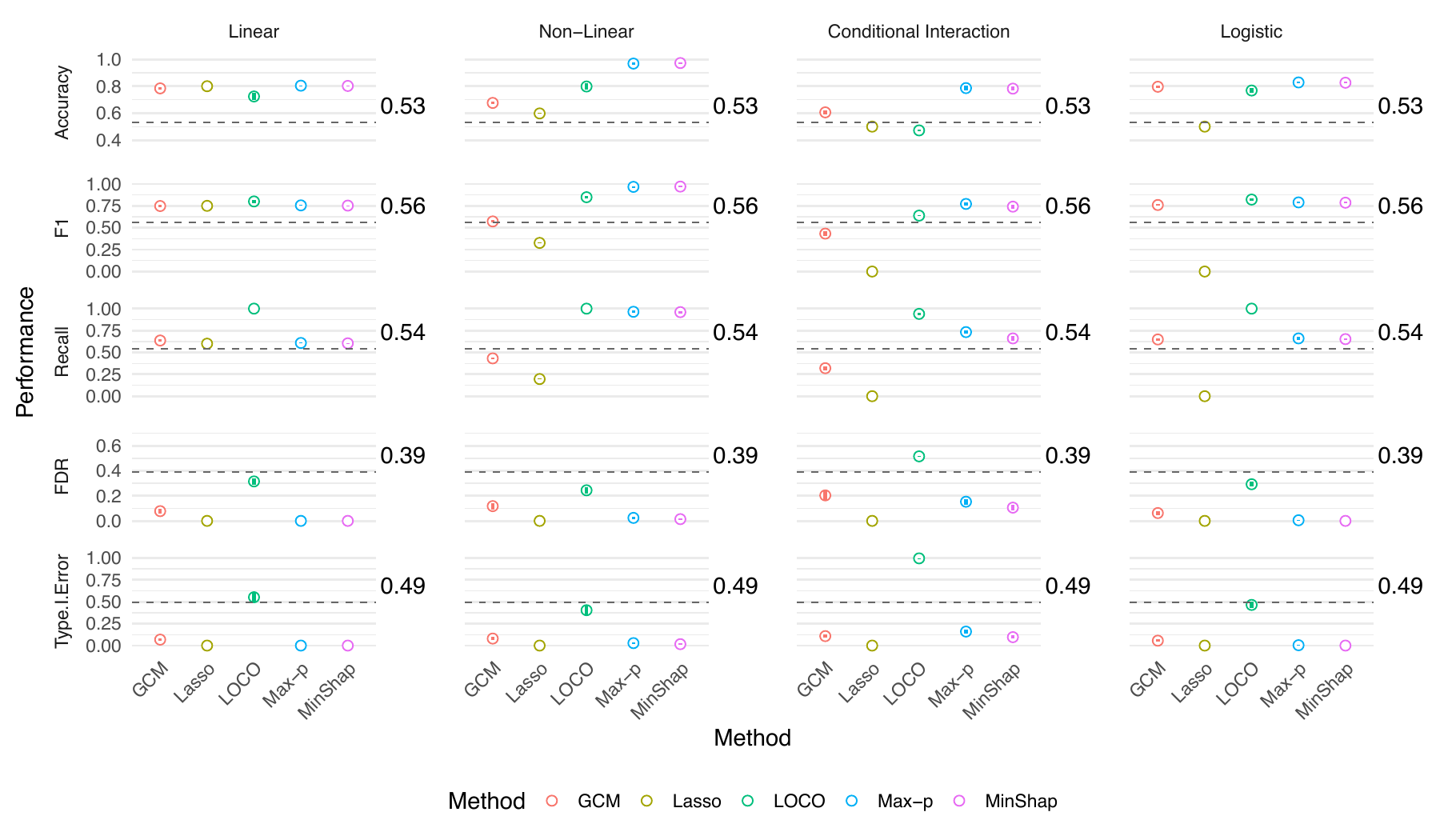}
    \caption{Average performance measure with standard deviation for model (a)-(d) for feature selection comparison using Neural Network. We see that MinShap and Max-p outperform other feature selection methods and have well-controlled Type I error in the meanwhile.}
    \label{fig:all_perf_nn}
\end{figure}

\begin{table}[h]
\centering
\caption{Jaccard Index for the stability of feature selection methods on model (a)--(d) using Neural Network.}
\label{tb:jaccard_nn}
\begin{small}
\begin{tabular}{@{}lccccc@{}}
\toprule
& \multicolumn{5}{c}{Feature Selection Method} \\
\cmidrule(lr){2-6}
Model & MinShap & Max-p & LOCO & GCM & Lasso \\
\midrule
(a)  & 0.99 & 0.98 & 0.73& 0.79&1.0 \\
(b)&  0.91& 0.89 &0.72 &0.72 &0.96\\
(c)&0.57&0.49&0.94 &0.38 & 1.0\\
(d)  & 0.89&0.87 &0.72 & 0.80&1.0\\
\bottomrule
\end{tabular}
\end{small}
\vskip -0.1in
\end{table}
The overall performance comparison is presented in Fig.\ref{fig:all_perf_nn}. MinShap and Max-p value approaches are much better than GCM, achieving higher accuracy and F1 score in all four models. Although the F1 score and accuracy for LOCO are similar to that of MinShap and Max-p in model (a) and (d), LOCO fails to control false positives; empirically, its Type I error is worse than that of a random selector with probability 0.5, and false discovery rate is higher than all other methods across all models. From Jaccard index Table \ref{tb:jaccard_nn}, we see that the stability of MinShap and Max-p are higher than LOCO and GCM in models (a), (b) and (d). The selection stability of LOCO in model (c) is high, but its inflated false positives undermine reliability. Similarly, Lasso remains stable, but it is not reliable in all nonlinear settings with low F1 score.

\section{Additional Real Data Analysis}\label{app:real}

In this section, we still mainly apply the variable selection methods of MinShap, Max-p, GCM and LOCO to additional real dataset. The dataset is analyzed using 5-fold cross-validation. Variable selection is carried out within each training fold, and the mean squared prediction error on the corresponding test fold is used to evaluate predictive performance. In addition, the stability of the feature selections is measured by Jaccard index.

\begin{table}[h]
\centering
\caption{Comparison of methods in terms of prediction error and selection stability on diabetes data under two models.}
\small
\renewcommand{\arraystretch}{1.2}
\begin{tabular}{llccc >{\centering\arraybackslash}p{5cm}}

\hline
\multicolumn{1}{c}{Model} &
\multicolumn{1}{c}{Method} &
\multicolumn{1}{c}{MSE} &
\multicolumn{1}{c}{se(MSE)} &
\multicolumn{1}{c}{Jaccard index} &
\multicolumn{1}{c}{Selected Features} \\
\hline

\multirow{4}{*}{XGBoost}
& MinShap & 3455.04 & 112.99 & 0.80 & [`bmi', `s5'] \\
& Max-p   & 3431.36 & 124.33 & 0.73 & [`bmi', `s5'] \\
& LOCO    & 3310.99 & 163.90 & 0.63 & [`bmi', `bp', `s5', `sex'] \\
& GCM     & 3228.06 & 130.00 & 0.81 & [`bmi', `bp', `s5', `sex'] \\
\hline

\multirow{4}{*}{\shortstack{Random\\forest}}
& MinShap & 3708.00 & 228.36 & 0.70 & [`bmi', `s5'] \\
& Max-p   & 3429.36 & 137.94 & 0.80 & [`bmi', `s5'] \\
& LOCO    & 3437.57 & 178.67 & 0.47 & [`bmi', `bp'] \\
& GCM     & 3294.31 & 133.62 & 0.85 & [`bmi', `bp', `s2', `s3', `s5', `sex'] \\
\hline
\end{tabular}
\vskip -0.1in
\label{dia_tb}
\end{table}
\subsection{Diabete data}
Understanding the link between clinical and biochemical measurements to disease progression is essential to medical decision-making. We aim to predict how patients' characteristics can be used to predict health outcomes with diabete dataset which is available in the \texttt{sklearn} Python library. The data has ten base features, age, sex, body mass index (BMI), average blood pressure (BP), and six blood serum measurements. Data is obtained for each of 442 diabetes patients, as well as the response of interest, a quantitative measure of disease progression one year after the time of measurement of the base features. From Table \ref{dia_tb}, we see that GCM always selects more variables than other methods, and the selection stability of LOCO is inferior to others. Features selected by MinShap and Max-p, "bmi" (body mass index) and "s5" (log of serum triglycerides level), using XGBoost and random forest are the most commonly selected in previous studies~\cite{bzdok_inference_2020, efron_least_2004}.

\end{appendices}
\end{document}